\documentclass[nohyperref]{article}

\usepackage[accepted]{icml2022}

\usepackage[textsize=tiny]{todonotes}

\usepackage[utf8]{inputenc} 
\usepackage[T1]{fontenc}    
\usepackage{hyperref}       
\usepackage{url}            
\usepackage{booktabs}       
\usepackage{amsfonts}       
\usepackage{nicefrac}       
\usepackage{microtype}      
\usepackage{xcolor}         

\usepackage{microtype}
\usepackage{graphicx}
\usepackage{booktabs} 

\usepackage{microtype}
\usepackage{graphicx}
\usepackage{booktabs} 

\usepackage{natbib}
\usepackage{amssymb, amsmath,amsthm}
\usepackage{amsfonts}
\usepackage{algorithm}
\usepackage{algorithmic}
\usepackage{paralist}
\usepackage{multirow}
\usepackage{times}
\usepackage{wrapfig}

\usepackage{url,enumerate}
\usepackage{color,xcolor}
\usepackage{makeidx}  
\usepackage{amsmath,amssymb}
\usepackage{mathtools}
\usepackage[small, compact]{titlesec}
\usepackage{xspace}
\usepackage{epstopdf}
\usepackage{mathrsfs}
\usepackage{times}
\usepackage{enumerate}
\usepackage{color}
\usepackage{graphicx,epsfig}
\usepackage{amsmath,amssymb,xspace}
\usepackage{url}
\usepackage{cleveref}

\usepackage{bm}
\usepackage{bbm}
\usepackage{upgreek}

\usepackage{multirow}
\usepackage{ulem}
\usepackage{cancel}
\usepackage{subcaption}
\usepackage{dsfont}
\usepackage{float}

\newtheorem{theorem}{Theorem}[section]

\newtheorem{lem}[theorem]{Lemma}
\newcommand{\ours}{NESH\xspace}
\newcommand{\gap}{$\Gamma$P\xspace}
\newcommand{\gaps}{$\Gamma$Ps\xspace}
\newcommand{\ggaps}{G$\Gamma$Ps\xspace}
\newcommand{\eg}{{\textit{e.g.},}\xspace}
\newcommand{\ie}{{\textit{i.e.},}\xspace}
\newcommand{\etc}{{\textit{etc}.}\xspace}

\icmltitlerunning{Nonparametric Embeddings of Sparse High-Order Interaction Events}

\begin{document}

\twocolumn[
\icmltitle{Nonparametric Embeddings of Sparse High-Order Interaction Events}

\icmlsetsymbol{equal}{*}

\begin{icmlauthorlist}
\icmlauthor{Zheng Wang}{equal,theU}
\icmlauthor{Yiming Xu}{equal,theUm}
\icmlauthor{Conor Tillinghast}{theUm}
\icmlauthor{Shibo Li}{theU}
\icmlauthor{Akil Narayan}{theUm,SciU}
\icmlauthor{Shandian Zhe}{theU}
\end{icmlauthorlist}

\icmlaffiliation{theU}{School of Computing, University of Utah}
\icmlaffiliation{theUm}{Department of Mathematics, University of Utah}
\icmlaffiliation{SciU}{Scientific Computing and Imaging (SCI) Institute, University of Utah}

\icmlcorrespondingauthor{Shandian Zhe}{zhe@cs.utah.edu}

\icmlkeywords{Machine Learning, ICML}
\vskip 0.3in
]


\printAffiliationsAndNotice{\icmlEqualContribution} 
\newcommand{\var}{{\rm var}}
\newcommand{\Tr}{^{\rm T}}
\newcommand{\vtrans}[2]{{#1}^{(#2)}}
\newcommand{\kron}{\otimes}
\newcommand{\schur}[2]{({#1} | {#2})}
\newcommand{\schurdet}[2]{\left| ({#1} | {#2}) \right|}
\newcommand{\had}{\circ}
\newcommand{\diag}{{\rm diag}}
\newcommand{\invdiag}{\diag^{-1}}
\newcommand{\rank}{{\rm rank}}
 \newcommand{\expt}[1]{\langle #1 \rangle}
\newcommand{\nullsp}{{\rm null}}
\newcommand{\tr}{{\rm tr}}
\renewcommand{\vec}{{\rm vec}}
\newcommand{\vech}{{\rm vech}}
\renewcommand{\det}[1]{\left| #1 \right|}
\newcommand{\pdet}[1]{\left| #1 \right|_{+}}
\newcommand{\pinv}[1]{#1^{+}}
\newcommand{\erf}{{\rm erf}}
\newcommand{\hypergeom}[2]{{}_{#1}F_{#2}}
\newcommand{\mcal}[1]{\mathcal{#1}}
\newcommand{\bepsilon}{\boldsymbol{\epsilon}}
\newcommand{\brho}{\boldsymbol{\rho}}
\renewcommand{\a}{{\bf a}}
\renewcommand{\b}{{\bf b}}
\renewcommand{\c}{{\bf c}}
\renewcommand{\d}{{\rm d}}  
\newcommand{\e}{{\bf e}}
\newcommand{\f}{{\bf f}}
\newcommand{\g}{{\bf g}}
\newcommand{\h}{{\bf h}}
\newcommand{\bi}{{\bf i}}
\newcommand{\bj}{{\bf j}} 

\renewcommand{\k}{{\bf k}}
\newcommand{\m}{{\bf m}}
\newcommand{\mhat}{{\overline{m}}}
\newcommand{\tm}{{\tilde{m}}}
\newcommand{\n}{{\bf n}}
\renewcommand{\o}{{\bf o}}
\newcommand{\p}{{\bf p}}
\newcommand{\q}{{\bf q}}
\renewcommand{\r}{{\bf r}}
\newcommand{\s}{{\bf s}}
\renewcommand{\t}{{\bf t}}
\renewcommand{\u}{{\bf u}}
\renewcommand{\v}{{\bf v}}
\newcommand{\w}{{\bf w}}
\newcommand{\x}{{\bf x}}
\newcommand{\y}{{\bf y}}
\newcommand{\z}{{\bf z}}
\newcommand{\A}{{\bf A}}
\newcommand{\B}{{\bf B}}
\newcommand{\C}{{\bf C}}
\newcommand{\D}{{\bf D}}
\newcommand{\E}{{\bf E}}
\newcommand{\F}{{\bf F}}
\newcommand{\G}{{\bf G}}
\newcommand{\Gcal}{{\mathcal{G}}}
\newcommand{\Dcal}{\mathcal{D}}
\renewcommand{\H}{{\bf H}}
\newcommand{\I}{{\bf I}}
\newcommand{\J}{{\bf J}}
\newcommand{\K}{{\bf K}}
\renewcommand{\L}{{\bf L}}
\newcommand{\Lcal}{{\mathcal{L}}}
\newcommand{\M}{{\bf M}}
\newcommand{\Mcal}{{\mathcal{M}}}
\newcommand{\Ocal}{{\mathcal{O}}}
\newcommand{\Fcal}{{\mathcal{F}}}
\newcommand{\Acal}{{\mathcal{A}}}
\newcommand{\Bcal}{{\mathcal{B}}}
\newcommand{\Ccal}{{\mathcal{C}}}
\newcommand{\Ecal}{{\mathcal{E}}}
\newcommand{\N}{\mathcal{N}}  
\newcommand{\bupeta}{\boldsymbol{\upeta}}
\renewcommand{\O}{{\bf O}}
\renewcommand{\P}{{\bf P}}
\newcommand{\Q}{{\bf Q}}
\newcommand{\R}{{\bf R}}
\renewcommand{\S}{{\bf S}}
\newcommand{\Scal}{{\mathcal{S}}}
\newcommand{\T}{{\bf T}}
\newcommand{\Tcal}{{\mathcal{T}}}
\newcommand{\U}{{\bf U}}
\newcommand{\Ucal}{{\mathcal{U}}}
\newcommand{\tU}{{\tilde{\U}}}
\newcommand{\tUcal}{{\tilde{\Ucal}}}
\newcommand{\V}{{\bf V}}
\newcommand{\W}{{\bf W}}
\newcommand{\Wcal}{{\mathcal{W}}}
\newcommand{\Vcal}{{\mathcal{V}}}
\newcommand{\X}{{\bf X}}
\newcommand{\KL}{\text{KL}}
\newcommand{\Xcal}{{\mathcal{X}}}
\newcommand{\Y}{{\bf Y}}
\newcommand{\Ycal}{{\mathcal{Y}}}
\newcommand{\Z}{{\bf Z}}
\newcommand{\Zcal}{{\mathcal{Z}}}
\newcommand{\hbeta}{\widehat{\bbeta}}
\newcommand{\homega}{\widehat{\bomega}}

\newcommand{\bfLambda}{\boldsymbol{\Lambda}}

\newcommand{\bsigma}{\boldsymbol{\sigma}}
\newcommand{\balpha}{\boldsymbol{\alpha}}
\newcommand{\bpsi}{\boldsymbol{\psi}}
\newcommand{\bphi}{\boldsymbol{\phi}}
\newcommand{\bPhi}{\boldsymbol{\Phi}}
\newcommand{\bbeta}{\boldsymbol{\beta}}
\newcommand{\boldeta}{\boldsymbol{\eta}}
\newcommand{\bnu}{\boldsymbol{\nu}}
\newcommand{\btau}{\boldsymbol{\tau}}
\newcommand{\bvarphi}{\boldsymbol{\varphi}}
\newcommand{\bzeta}{\boldsymbol{\zeta}}

\newcommand{\blambda}{\boldsymbol{\lambda}}
\newcommand{\bLambda}{\mathbf{\Lambda}}

\newcommand{\btheta}{\boldsymbol{\theta}}
\newcommand{\bpi}{\boldsymbol{\pi}}
\newcommand{\bxi}{\boldsymbol{\xi}}
\newcommand{\bSigma}{\boldsymbol{\Sigma}}
\newcommand{\bPi}{\boldsymbol{\Pi}}
\newcommand{\bOmega}{\boldsymbol{\Omega}}
\newcommand{\bomega}{\boldsymbol{\omega}}

\newcommand{\bx}{{\bf x}}
\newcommand{\bgamma}{\boldsymbol{\gamma}}
\newcommand{\bGamma}{\boldsymbol{\Gamma}}
\newcommand{\bUpsilon}{\boldsymbol{\Upsilon}}
\newcommand{\talpha}{\widetilde{\alpha}}
\newcommand{\tbeta}{{\widetilde{\bbeta}}}
\newcommand{\ttbeta}{{\widetilde{\beta}}}
\newcommand{\tomega}{{\widetilde{\bomega}}}
\newcommand{\ttomega}{{\widetilde{\omega}}}
\newcommand{\tlam}{{\tilde{\lambda}}}
\newcommand{\tblam}{{\tilde{\blambda}}}
\newcommand{\hM}{{\widehat{M}}}
\newcommand{\hA}{{\widehat{A}}}
\newcommand{\MN}{{\mathcal{M}\N}}

\newcommand{\bmu}{\boldsymbol{\mu}}
\newcommand{\1}{{\bf 1}}
\newcommand{\0}{{\bf 0}}

\newcommand{\bs}{\backslash}
\newcommand{\ben}{\begin{enumerate}}
\newcommand{\een}{\end{enumerate}}
 \newcommand{\notS}{{\backslash S}}
 \newcommand{\nots}{{\backslash s}}
 \newcommand{\noti}{{\backslash i}}
 \newcommand{\notj}{{\backslash j}}
 \newcommand{\nott}{\backslash t}
 \newcommand{\notone}{{\backslash 1}}
 \newcommand{\nottp}{\backslash t+1}

\newcommand{\notk}{{^{\backslash k}}}
\newcommand{\notij}{{^{\backslash i,j}}}
\newcommand{\notg}{{^{\backslash g}}}
\newcommand{\wnoti}{{_{\w}^{\backslash i}}}
\newcommand{\wnotg}{{_{\w}^{\backslash g}}}
\newcommand{\vnotij}{{_{\v}^{\backslash i,j}}}
\newcommand{\vnotg}{{_{\v}^{\backslash g}}}
\newcommand{\half}{\frac{1}{2}}
\newcommand{\msgb}{m_{t \leftarrow t+1}}
\newcommand{\msgf}{m_{t \rightarrow t+1}}
\newcommand{\msgfp}{m_{t-1 \rightarrow t}}

\newcommand{\proj}[1]{{\rm proj}\negmedspace\left[#1\right]}
\newcommand{\argmin}{\operatornamewithlimits{argmin}}
\newcommand{\argmax}{\operatornamewithlimits{argmax}}

\newcommand{\dif}{\mathrm{d}}
\newcommand{\abs}[1]{\lvert#1\rvert}
\newcommand{\norm}[1]{\lVert#1\rVert}

\newcommand{\mrm}[1]{\mathrm{{#1}}}
\newcommand{\RomanCap}[1]{\MakeUppercase{\romannumeral #1}}
\newcommand{\EE}{\mathbb{E}}
\newcommand{\cmt}[1]{}

\begin{abstract} 
High-order  interaction events are common in real-world applications. Learning embeddings that encode the complex relationships of the participants from these events is of great importance in knowledge mining and predictive tasks. Despite the success of existing approaches, \eg Poisson tensor factorization, they ignore the sparse structure underlying the data, namely the occurred interactions are far less than the possible interactions among all the participants. In this paper, we propose Nonparametric Embeddings of Sparse High-order interaction events (\ours). We hybridize a sparse hypergraph (tensor) process and a matrix Gaussian process to capture both the asymptotic structural sparsity within the interactions and nonlinear temporal relationships between the participants. We prove strong asymptotic bounds (including both a lower and an upper bound) of the sparsity ratio, which reveals the asymptotic properties of the sampled structure. We use batch-normalization, stick-breaking construction and sparse variational GP approximations to develop an efficient, scalable model inference algorithm. We demonstrate the advantage of our approach in several real-world applications.
\end{abstract}
		
\section{Introduction}
Many real-world applications are filled with interaction events between multiple entities or objects, \eg the purchases happened among \textit{customers}, \textit{products} and \textit{shopping pages} at 
\texttt{Amazon.com}, and tweeting between \textit{twitter users} and \textit{messages}. Embedding these events, namely, learning a representation of the participant objects to encode their complex relationships, is of great importance and interest, in discovering hidden patterns from data, \eg clusters and outliers, and performing downstream tasks, such as recommendation and online advertising. 

While Poisson tensor factorization is a popular framework for the representation learning of those events, current methods, \eg~\citep{chi2012tensors, HaPlKo15, Hu2015CountTensor,schein2015bayesian, Schein:2016:BPT:3045390.3045686,schein2019poisson}, are mostly based on a multilinear factorization form, \eg~\citep{Tucker66,Harshman70parafac}, and therefore might be inadequate to estimate complex, nonlinear temporal relationships in data.
More important,  existing methods overlook the structural sparsity underlying these events. That is, the observed interactions are far less than all the possible interactions among the participants (\eg 0.01\%). Many factorization models rely on tensor algebras and demand all the tensor entries (\ie interactions) should be observed~\citep{Kolda09TensorReview,kang2012gigatensor,choi2014dfacto}. Even for those entry-wise factorization models ~\citep{RaiDunson2014,zhao2015bayesian,du2018probabilistic}, from the Bayesian viewpoint, they are equivalent to first generating the entire tensor and then marginalizing out the unobserved entries. In practice, however, the observed interactions are often very sparse, and their proportion can get even smaller with the increase of objects. For example, in online shopping,  with the growth of users and items, the number of actual purchases (while growing) takes a smaller percentage of all possible purchases, i.e., all (user, item, shopping-page) combinations, because the latter grows much faster.

In this paper, we propose \ours, a novel nonparametric Poisson factorization approach for  high-order interaction events embedding. Not only does \ours flexibly estimate various nonlinear temporal  relationships of the participants, it also can capture structural sparsity  within the present interactions, absorbing both the structural traits and hidden relationships into the embeddings. Our major contributions are the following:
\begin{itemize}
	\item Model. We hybridize the recent sparse tensor (hypergraph) processes (STP)~\citep{tillinghast2021nonparametric} and  matrix Gaussian processes (MGP) to develop a sparse event model, where the embeddings are in charge of  both generating the interactions and modulating the event rates, hence can jointly encode the temporal relationships and sparse structure knowledge.
	\item Theory. We use Poisson tail estimate, Bernstein’s inequality and L'H\^opital's rule to prove strong asymptotic bounds of the sparsity ratio, including both a lower and upper bound. The prior work \citet{tillinghast2021nonparametric} only shows the sparsity ratio of the sampled tensors asymptotically converges to zero, yet never gives an estimate of the convergence rate. Our new result reveals more theoretical insight of  STP in producing sparse structures, and can also characterize the classical sparse graph generation models~\citep{caron2014sparse,williamson2016nonparametric}.
	\item Algorithm. We use the stick-breaking construction of the normalized hypergraph process to compute the embedding prior, and then use batch-normalization and variational sparse GP framework to develop an efficient and scalable model estimation algorithm.
\end{itemize}

 For evaluation, we conducted simulations to demonstrate that our theoretical bounds can indeed match the actual sparsity ratio and capture the asymptotic trend. Hence they can provide a reasonable convergence rate estimate and characterize the behavior of the prior. We then tested our approach \ours  on three real-world datasets. \ours achieves much better predictive performance than the existing methods  that  use Poisson tensor factorization, additional time steps, local time dependency windows and triggering kernels. \ours also outperforms the same model with the sparse hypergraph prior removed, which demonstrates the importance of accounting for the structure sparsity. 
 We then looked into the embeddings estimated by \ours, and found interesting patterns, including the clusters of users, sellers, and item categories in online shopping, and groups of states where car crash accidents happened. 
\section{Background}
We assume that we observed $K$-way interactions  among $K$ types of objects or entities (\eg \textit{customers}, \textit{products} and \textit{sellers}). We denote by $D_k$ the number of objects of type $k$, and index each object by $i_k$ ($1 \le i_k \le D_k$). We then index a particular interaction by a tuple $\bi = (i_1, \ldots, i_K )$.  We may observe multiple occurrences of a particular interaction. We denote the sequence of these events by $\s_{\bi} = [s_{\bi1}, \ldots, s_{\bi m_{\bi}}]$ where $s_{\bi j}$ is the time-stamp when $j$-th event occurred ($1 \le j \le m_\bi$) and $m_{\bi}$ is the total number of the occurrences of $\bi$. Suppose we have observed events of a collection of interactions, $\Scal = \{\s_{\bi_1}, \ldots, \s_{\bi_N}\}$, we aim to learn an embedding for each participant object. Note that one object may participate in multiple, distinct interactions. We denote by $\u^k_j$ the embeddings for object $j$ of type $k$, which is an $R$ dimensional vector. We stack the embeddings of the objects of type $k$ into a embedding matrix $\U^k = [\u^k_1, \ldots, \u^k_{D_k}]^\top$, and denote by $\Ucal = \{\U^1, \ldots, \U^K\}$ all the embedding matrices. 

To estimate the embeddings from $\Scal$, a popular approach is tensor factorization. We can introduce a $K$-mode tensor $\Ycal \in  \mathbb{R}^{D_1 \times \ldots \times D_K}$ accordingly, where each mode $k$ includes $D_k$ objects, and each entry $\bi$ corresponds to an event sequence $\s_{\bi}$. For event modeling, we can use the popular (homogeneous) Poisson processes, and the probability of $\s_{\bi}$ is given by 
\begin{align}
	p(s_{\bi}|\lambda_{\bi}) = e^{-\int_0^T \lambda_{\bi} \d t} \prod\nolimits_{j=1}^{m_{\bi}} \lambda_{\bi} = e^{-T\lambda_{\bi}} \lambda_{\bi}^{m_\bi}, \label{eq:hpp}
\end{align} 
where $T$ is the total time span of all the event sequences, and $\lambda_{\bi} > 0$ is the rate (or intensity) of the interaction $\bi$. Since the probability is only determined by the event count, we can place the count value $m_\bi$ in the entry $\bi$ of $\Ycal$, and perform count tensor factorization.  Classical tensor factorization approaches include Tucker decomposition~\citep{Tucker66}, CANDECOMP/PARAFAC (CP) decomposition~\citep{Harshman70parafac}, \etc  Tucker decomposition assumes $\Ycal = \Wcal \times_1 \U^{1} \times_2 \ldots \times_K \U^{K}$, where  $\mathcal{W} \in \mathbb{R}^{r_1 \times \ldots \times r_K}$ is a parametric core tensor,  $\{\U^k|1\le k \le K\}$ are embedding matrices, and $\times_k$ is  the tensor-matrix product at mode $k$~\citep{kolda2006multilinear}, which is very similar to the matrix-matrix product. If we set all $r_k = R$, and constrain $\Wcal$ to be diagonal, Tucker decomposition is reduced to  CANDECOMP/PARAFAC (CP) decomposition~\citep{Harshman70parafac}.  While numerous tensor factorization algorithms have  been developed, \eg  ~\citep{Chu09ptucker,kang2012gigatensor,choi2014dfacto}, most of them inherit the CP or Tucker form. 
To perform count tensor factorization, we can use Poisson process likelihood \eqref{eq:hpp} for each entry, and apply the Tucker/CP decomposition to the rates $\{\lambda_{\bi}\}$ or log rates  $\{\log(\lambda_{\bi})\}$ ~\citep{chi2012tensors, Hu2015CountTensor}.   A more refined strategy is to further partition the events  into a series of time steps, \eg by weeks or months,  augment the count tensor with a time-step mode~\citep{xiong2010temporal, schein2015bayesian, Schein:2016:BPT:3045390.3045686,schein2019poisson}, and jointly estimate the embeddings of these steps, $\{\s_1, \s_2, \ldots\}$. We can also model the dependencies between the time steps with some dynamics, \eg ~\citep{xiong2010temporal}.


\section{Model}

Despite the success of existing Poisson tensor factorization approaches, they might be restricted in that (1) the commonly used CP/Tucker factorization over the rates are multilinear to the embeddings and therefore cannot capture more complex, nonlinear relationships between the interaction participants; (2) the homogeneous assumption, \ie constant event rate, might be oversimplified, overlook temporal variations of the rates, and hence miss critical temporal patterns. More important, (3) in many real-world applications, the present interactions are very sparse, when contrasted to all possible interactions.  For example,  despite the massive online transactions, the ratio between the number of actual transactions and \textit{all} possible transactions (\ie all combinations of (\textit{customer, product, seller}) is tiny\footnote{see  Amazon data samples (\url{http://jmcauley.ucsd.edu/data/amazon/}) and dataset information in our experiments in Sec. \ref{sect:exp}.  }, slightly above zero (\eg $0.01\%$). This proportion can get even smaller with the growth of customers, products and sellers, because their combinations can grow much faster.  Similar observations can be found in clicks in online advertising, message tweeting, \etc  Existing methods, however, are not aware of this data sparsity, and lack an effective modeling framework to embed the underlying sparse structures. To overcome these limitations, we propose \ours, a novel  nonparametric embedding model for sparse high-order interaction events, presented as follows. 
\subsection{Nonparametric Sparse Event Modeling for High-Order Interactions}
First, to highlight the sparse structure within the observed events, we view the participants as nodes, and their interactions as $K$-way hyperedges in a hypergraph. Each edge connects $K$ participants (nodes), corresponding to a particular interaction $\bi$. Attached to $\bi$ is a sequence of events $\s_\bi$ --- the occurrence history of $\bi$.  Our goal is to learn an embedding for each node, which is able to not only estimate the complex temporal relationships between the nodes, but also capture the traits of the sparse hypergraph structure. To this end, we follow~\citep{tillinghast2021nonparametric,caron2014sparse} to construct a stochastic process to sample the hypergraph, with a guarantee of sparsity in the asymptotic sense. Specifically, for each node type  $k (1\le k \le K)$, we sample a set of Gamma processes~\citep{hougaard1986survival,brix1999generalized} to represent an infinite number of nodes  and their weights, 
\begin{align}
	W^\alpha_{k,r} \sim \Gamma\text{P}(\beta_\alpha) \;\; (1 \le k \le K, 1 \le r \le R) \label{eq:gaps}
\end{align}
where $\beta_\alpha$ is a Lebesgue base measure confined to $[0, \alpha] (\alpha > 0)$.  Next, we use these \gaps to construct a product-measure sum, with which as the mean measure to sample a Poisson point process  (PPP)~\citep{kingman1992poisson}, which represents the sampled edges of the hypergraph, 
\begin{align}
	&M = \sum\nolimits_{r=1}^R W^\alpha_{1,r} \times \ldots \times W^\alpha_{K,r}, \,\, \notag \\ &T | \{W^\alpha_{k, r}\}_{1 \le k \le K, 1 \le r \le R} \sim \text{PPP}(M). \label{eq:ppp}
\end{align}  
Accordingly, $T$ has the following form, 
 \[
 T = \sum_{\bi \in \Ecal} c_{\bi} \cdot \delta_{\Theta^\alpha_\bi},
 \]
 where each point represents an hyperedge (interaction), $\Ecal$ is the set of all the sampled points, $c_{\bi}>0$ is the count of the point $\bi$, and $\Theta^\alpha_{\bi} = \{(\theta^\alpha_{1i_1}, \ldots, \theta^\alpha_{Ki_K})\}$ represents the location of that point and comes from the \gaps, and $\delta_{[\cdot]}$ is the Dirac measure.   In essence, the \gaps sample infinite nodes for each type $k (1\le k \le K)$, and then the PPP picks the nodes from each type to sample the hyperedges, \ie multiway interactions. 
 \cmt{
 \begin{lem}
 	For any fixed $\alpha>0$, the number of sampled hyperedges $N^\alpha=|\Ecal|$ via \eqref{eq:gaps}\eqref{eq:ppp} is finite almost surely (\ie with probability one). When $\alpha \rightarrow \infty$, $N^\alpha \rightarrow \infty$ a.s.
 	\label{lem:1}
 \end{lem}
 To examine the sparsity, we are interested in the nodes connected by the sampled hyperedges $\Ecal$, which we refer to as ``active'' nodes. They are participants of the observed interactions.  For example, if an edge $2-3-1$ is sampled, then node 2, 3, 1 (of type 1, 2, 3 respectively) are active nodes. Denote by $D_k^\alpha$ the number of distinct active nodes of type $k$.  If we connect all the active nodes of the $K$ types, we will have  $\prod_{k=1}^K D_k^\alpha$ hyperedges (interactions). Intuitively, sparsity means the proportion of the sampled edges in all possible edges is very small, and the former grows slower than the latter, with the increase of active nodes. This is guaranteed by 
}

To examine the sparsity, we look into the nodes in the sampled hyperedges $\Ecal$, which are referred to as ``active'' nodes. They are participants of the interactions.  For example, if an edge $2-3-1$ is sampled, then node 2, 3, 1 (of type 1, 2, 3 respectively) are active nodes. Denote by $D_k^\alpha$ the number of distinct active nodes of type $k$.  If we connect all the active nodes of the $K$ types, we will have  $\prod_{k=1}^K D_k^\alpha$ hyperedges (interactions) in total, \ie the volume. Sparsity means the proportion of the sampled edges in all possible edges is very small, and the former grows slower than the latter, with the increase of active nodes. Denote by $N^\alpha$ the number of sampled edges. The sparsity is guaranteed by
 \begin{lem}[Corollary 3.1.1 \citep{tillinghast2021nonparametric}]
 	  $N^\alpha = o(\prod _{k=1}^K D^\alpha_k)$ almost surely as $\alpha \rightarrow \infty$, \ie  
 	$\underset{\alpha \rightarrow \infty}{\lim } \frac{N^\alpha}{\prod _{k=1}^K D^\alpha_k} = 0 \;\;\;a.s. $
 	\label{lem:2}
 \end{lem}
Note that this is an asymptotic notion of structural sparsity --- with the hyper-graph volume growing (\ie increasing $\alpha$), the proportion of the sampled edges is tending to zero. It is different from other notions~\citep{choi2014dfacto,hu2015zero} where the sparsity means data is dominated by zero values.  

Given the sparse hypergraph prior, we then sample the observed edges (interactions) and associated events, $\Dcal = \{(\bi_1,\s_{\bi_1}), \ldots, (\bi_N, \s_{\bi_N})\}$. Since they are always finite, we 
can use the standard PPP construction~\citep{kingman1992poisson}  to sample these observations, which \cmt{is equivalent to \eqref{eq:ppp} yet }is computationally much more convenient and efficient. Specifically, we normalize the mean measure $M = \sum_{r=1}^{R} W^\alpha_{1,r} \times \ldots \times W^\alpha_{K,r}$ in \eqref{eq:ppp} to obtain a probability measure, and use it to sample the  $N$ points (\ie edges/interactions) independently. To normalize $M$, we need to first normalize each \gap $W^\alpha_{k,r} (1 \le k \le K, 1 \le r \le R )$, which gives a Dirichlet process  (DP)~\citep{ferguson1973bayesian}, with the strength as $\beta_\alpha([0, \alpha]) = \alpha$, and  base measure as the normalized base measure of $W^\alpha_{k,r}$ that is  a uniform distribution in $[0, \alpha]$, 	
\begin{align}
	G^k_r &\sim \text{DP}\big(\alpha, \text{Uniform}([0, \alpha]) \big),  \label{eq:dp}
\end{align}
where $1\le k \le K$ and $1 \le r \le R$. The normalized $M$ is $\frac{1}{R} \sum_{r=1}^{R} G^1_{r} \times \ldots \times G^K_{r}$. To capture rich structural information,  we follow ``Model-II'' in~\citep{tillinghast2021nonparametric} to sample multiple DP weights for each node. Specifically, we drop the locations, and only sample the weights, which follow the GEM distribution~\citep{griffiths1980lines,engen1975note,mccloskey1965model}, and obtain
\begin{align}
	\widehat{G}_r^k = \sum_{j=1}^\infty \omega^k_{rj}\cdot  \delta_{j}.  \label{eq:dp-sample}
\end{align}
Accordingly, we construct a probability measure over all possible edges (interactions), 
\begin{align}
	\hM = \sum_{\bi=(1, \ldots, 1)}^{(\infty, \ldots, \infty)}  w_{\bi} \cdot \delta_{\bi}, \label{eq:entry-dist}
\end{align}
where  $w_{\bi} = \frac{1}{R} \sum_{r=1}^{R} \prod_{k=1}^K \omega_{ri_{k}}^k$.
We then sample each observed interaction $\bi_n \sim \hM$, and  the probability is
\begin{align}
	p(\Ecal) = \prod\nolimits_{n=1}^N p(\bi_n) = \prod\nolimits_{n=1}^N w_{\bi_n}.  \label{eq:edge-prob}
\end{align}

Now, it can be seen that from \eqref{eq:dp} and \eqref{eq:dp-sample}, for each node $j$ of type $k$, we have sampled a set of $R$ weights $ \{\omega^k_{1j}, \ldots, \omega^k_{Rj}\}$ from $R$ DPs.  From \eqref{eq:entry-dist}, we can see  these weights reflect the activity of the node interacting with other nodes (of different types). Each weight naturally represents the sociability in one community/group, and these communities are overlapping. We use these sociabilities  to construct the embeddings of the nodes. Therefore, they encode the sparse structural information underlying the observed interactions \footnote{Model-II in~\citep{tillinghast2021nonparametric} actually has made an additional adjustment on top of \eqref{eq:ppp}. According to the superposition theorem, $\text{PPP}(\sum_{r=1} ^RW_{1,r}^\alpha \times \ldots \times W_{K,r}^\alpha) \overset{D}{=} \sum_{r=1}^R \text{PPP}(W_{1, r}^\alpha \times \ldots \times W_{K,r}^\alpha)$. It means \eqref{eq:ppp} essentially samples $R$ hypergraphs independently and places them together. Model-II further performs a probabilistic merge of the $R$ hypergraphs. To see this,  from \eqref{eq:dp-sample} the nodes of each type in each hypergraph are indexed by the same set of integers  ($1, 2, 3, \ldots$), and so all the possible edges in each hypergraph are indexed by the same set of index tuples.  From \eqref{eq:entry-dist} and \eqref{eq:edge-prob},  the probability of sampling a particular edge indexed by $\bi$ is $\omega_\bi = \frac{1}{R}\sum_{r=1}^R \prod_{k=1}^K \omega_{ri_k}^k$. This can be explained as the following merging procedure. We randomly select one hypergraph (with probability $\frac{1}{R}$) and check if edge $\bi$ has been sampled in that hypergraph. If it has, we add edge $\bi$ in the new graph; otherwise, we do not add the edge. Since in each hypergraph $r$, the probability of edge $\bi$ being sampled is $\prod_{k=1}^K \omega_{ri_k}^k$,  the overall probability of sampling the edge $\bi$ in the new graph is the average, $\omega_\bi = \frac{1}{R}\sum_{r=1}^R \prod_{k=1}^K \omega_{ri_k}^k$. It is trivial to see that the merged hypergraph is still asymptotically sparse: since these hypergraphs can be viewed as independently sampled based on the same set of nodes, each of which is asymptotically sparse,  their summation is also asymptotically sparse. 
	The benefit  is that since we align these hypergraphs via the integer indices of the nodes, we can assign multiple sociabilities for each node to be better able  to capture the abundant structural information.
}

Given the sampled interactions $\Ecal$, we then sample their occurred events  $\Scal = [\s_{\bi_1}, \ldots, \s_{\bi_N}]$. To flexibly capture the temporal patterns, we use non-homogeneous Poisson processes.  For each observed interaction $\bi_n$, we consider a raw rate function $\rho_{\bi_n}(t)$, and then link it to a positive rate function by taking the square, $\lambda_{\bi_n}(t) = \left(\rho_{\bi_n}(t)\right)^2$. Note that we can also use $\exp(\cdot)$, which, however, performs worse in our experiments.  In order to capture the complex relationships of the  rate functions and their temporal variations, we jointly sample  the collection of rate functions, $\brho = \{\rho_{\bi_n}(t)|1\le n \le N\}$, from a matrix Gaussian process (MPG)~\citep{Rasmussen06GP}, 
\begin{align}
	\brho \sim \MN(\0; \kappa_1(\x_{\bi}, \x_{\bi'}), \kappa_2(t, t')), \label{eq:mgp}
\end{align}
where $\kappa_1(\cdot, \cdot)$ is the (row) covariance function across different  interactions (hyperedges), the inputs are the embeddings of participant nodes, $\x_\bi = [\u^1_{i_1}; \ldots; \u^K_{i_K}]$, and $\kappa_2(\cdot, \cdot)$ is the (column) covariance function about the time.\cmt{Hence,  $\forall \bi, \bi', t, t'$, $\text{cov}(\rho_{\bi}(t), \rho_{\bi'}(t')) = \kappa_1(\x_{\bi}, \x_{\bi'}) \kappa_2(t, t')$.} We can choose nonlinear kernels for $\kappa_1$ and $\kappa_2$ to capture the complex relationships and temporal dependencies within $\brho$. 
Given the rate functions, we then sample the observed event sequences from 
\begin{align}
	&p(\Scal|\blambda) = \notag \\ &\prod_{n=1}^N \exp(-\int_0^T \left(\rho_{\bi_n}(t)\right)^2\d t) \prod_{j=1}^{m_{\bi_n}} \left(\rho_{\bi}(s_{\bi_n j})\right)^2, \label{eq:value-prob}
\end{align}
where $T$ is the total time span across all the event sequences.

From \eqref{eq:edge-prob} and \eqref{eq:value-prob}, we can see that, via coupling the DPs and matrix GP, both the structural properties in the sparsely observed interactions and hidden temporal relationships of the participant nodes in the events can be grasped and absorbed into the embeddings.   

\subsection{Theoretical Analysis of Sparsity}
Although \citet{tillinghast2021nonparametric} has proved the asymptotic sparsity guarantee of the hypergraph process in \eqref{eq:ppp} (referred to as sparse tensor process in their paper),  \ie Lemma \ref{lem:2}, the conclusion is rough in that we have no idea how the sparsity of the sampled hyper-graph varies along with more and more active nodes. We only know that at the limit, the sparsity ratio becomes zero. While~\citet{caron2014sparse} gave some convergence rate estimate in their binary graph generating models under a similar modeling framework, the estimate is only available when using generalized \gaps (\ggaps)~\citep{hougaard1986survival} with a particular parameter range (see Theorem 10 in their paper).  The estimate is not available for the popular ordinary \gaps as in our model.  \ggaps cannot be normalized as DPs and are much harder/inconvenient for computation and inference. To extract more theoretical insight, we prove asymptotic bounds of the sparsity ratio for our hyper-graph process, which not only deepen our understanding of the properties of the sampled structures, but also fill the gap of prior works. 
\begin{lem}
	For a sparse hyper-graph process defined as in \eqref{eq:ppp}, for all sufficiently large $\alpha$, there exists an absolute constant $C>0$ such that, with probability at least $1-(C\alpha)^{-K}$,   
	\begin{align*}
		&\frac{e^{-1.03(2K)^{1/K}K(\log \alpha)^{1/K}}}{2K\log \alpha}\cdot\left[\frac{1.82}{(K-1)\log (1.01 \alpha)}\right]^K\notag \\
		&\leq \frac{N^\alpha}{\prod_{k=1}^{K}D_k^\alpha} \leq \left[\frac{2.11}{(K-1)\log (0.99 \alpha)}\right]^K.
	\end{align*}
\label{lem:lem2}
\end{lem}
\textit{Proof sketch}. We first use concentration inequalities, including Poisson tail estimate~\citep{vershynin2018high} and Bernstein’s inequality, and L'Hopital's rule to  bound the measure of each \gap on $[0, \alpha]$, and then take a union bound over $k=1\ldots K$ to obtain an upper bound of $N^\alpha$. The lower bound is more technical, and requires a careful estimate of the support of the intensity measure that appears in the sampled entries with high probability, for which we apply a novel probabilistic argument. We mainly combine Poisson tail estimates,  union bounds, the Bernoulli distribution and L'H\^opital's rule to bound each $D_k^\alpha$ and then derive the lower bound of $N^\alpha$. We leave the details in Appendix.
\section{Algorithm}
The matrix GP in our model, coupled with DPs, is computational costly. When the number of present interactions and/or the number of their events are large, we will have to compute a huge row and/or column covariance matrix, their inverse and  determinants, which is very expensive or even infeasible.  To address these issues, we use the stick-breaking construction~\citep{sethuraman1994constructive}, sparse variational GP approximation~\citep{titsias2009variational,hensman2013gaussian} and batch normalization~\citep{ioffe2015batch} to develop an efficient, scalable variational inference algorithm. 

Specifically, we use the stick-breaking construction to sample the DP weights (or GEM distribution), 
\begin{align}
	&v^k_{rj} \sim \text{Beta}(1, \alpha), \;\; \omega^k_{rj} = v^k_{rj} \prod\nolimits_{l=1}^{j-1} (1 - v^k_{rl}),\;\;\;  \label{eq:stick-breaking}
\end{align}
where $1\le j \le \infty$. Therefore, we only need to estimate the stick-breaking variables $\{v^k_{rj}\}$, from which we can outright calculate the weights (or sociabilities). Since these weights can be very small and close to zero, we use their logarithm to construct the embedding of each node $j$ of type $k$, $\u^k_j = [\log(\omega^k_{1j}); \ldots; \log(\omega^k_{Rj})]$.

Next, to conveniently handle the matrix GP prior in \eqref{eq:mgp}, we unify all the raw rate functions as one function of the embeddings and time, $\rho_{\bi}(t) = f(\x_{\bi}, t)$, over which we assign a GP prior with a product covariance (kernel) function, 
$\kappa([\x_{\bi}, t], [\x_{\bi'}, t']) = \kappa_1(\x_{\bi}, \x_{\bi'}) \kappa_2(t, t')$. This is computationally  equivalent to \eqref{eq:mgp} because they share the same covariance function. But we only need to deal with one function. Accordingly, the function values at the event time-stamps (across all the interactions), $\f = \{f(\x_{\bi_n}, s_{\bi_n j})\}_{n, j}$ follow a multivariate Gaussian distribution, 
\begin{align}
	p(\f|\Ucal) = \N(\f|\0, \kappa(\X, \X)),
\end{align}
where each row of $\X$ consist of the embedding $\x_{\bi_n}$ and a time-stamp $s_{\bi_n j}$. Combing with \eqref{eq:edge-prob}  \eqref{eq:value-prob} \eqref{eq:stick-breaking}, the joint probability of our model is 
\begin{align}
	&p(\{\v^k_j\}_{1\le j \le D_k, 1 \le k \le K}, \Ecal, \Scal, \f) \notag \\=& \prod_{k=1}^K \prod_{j=1}^{D_k}\prod_{r=1}^R \text{Beta}(v^k_{rj}|1, \alpha) \cdot \N\left(\f|\0, \kappa(\X, \X)\right) \label{eq:joint} \\
	\cdot &  \prod_{n=1}^N \omega_{\bi_n} \exp(-\int_0^T \left(f(\x_{\bi_n},t)\right)^2\d t) \prod_{j=1}^{m_{\bi_n}} \left(f(\x_{\bi_n}, s_{\bi_n j})\right)^2, \notag
\end{align}
where $\v^k_j = \{v^k_{jr}\}_{1\le r \le R}$. The stick-breaking variables associated with inactive nodes, \ie the nodes that do not participate any interactions, have been marginalized out. 

Next, to dispense with the huge covariance matrix in \eqref{eq:joint} for $\f$, we use the sparse variational GP approximation~\citep{hensman2013gaussian} to develop a variational inference algorithm. Specifically, we introduce a small set of pseudo inputs $\Z = [\z_1, \ldots, \z_h]^\top$ for $f(\cdot)$, where $h$ is far less than the dimension of $\f$. We then define the pseudo outputs $\b = [f(\z_1), \ldots, f(\z_h)]^\top$. We augment our model by jointly sampling $\{\f, \b\}$. Due to the GP prior over $f(\cdot)$, $\{\f, \b\}$  follow a multivariate  Gaussian distribution that can be decomposed as $p(\f, \b) = p(\b) p(\f|\b)$,
where $p(\b) = \N\big(\b|\0, \kappa(\Z, \Z)\big)$, $p(\f|\b) = \N(\f|\m_{f|b}, \bSigma_{f|b})$ is a conditional Gaussian distribution,   $\m_{f|b} = \kappa(\X, \Z)\kappa(\Z, \Z)^{-1}\b$ and $\bSigma_ {f|b} = \kappa(\X, \X) - \kappa(\X, \Z)\kappa(\Z, \Z)^{-1}\kappa(\Z, \X)$. The probability of the augmented model has the following form, 
\begin{align}
	p(\text{Joint}) = \text{OtherTerms} \cdot p(\b) p(\f|\b) p(\Scal|\f), \label{eq:joint-2}
\end{align}
where
$$p(\Scal|\f) =  \prod_{n=1}^N \exp(-\int_0^T (f(\x_{\bi_n},t))^2\d t) \prod_{j=1}^{m_{\bi_n}} (f(\x_{\bi_n}, s_{\bi_n j}))^2$$
 is the likelihood of the events. Compared to \eqref{eq:joint}, we just replace the Gaussian prior over $\f$ by the joint Gaussian prior over $\{\f, \b\}$. If we marginalize out $\b$, we will recover the original distribution \eqref{eq:joint}. Now, we construct a variational evidence lower bound (ELBO) to avoid computing the covariance matrix $\kappa(\X, \X)$. To this end, we introduce a variational  posterior for $\{\f, \b\}$,
$q(\f, \b) = q(\b)p(\f|\b)$,
where $q(\b) = \N(\b|\bmu, \L\L^\top)$, and $\L$ is a lower triangular matrix. Note that $\L\L^\top$ is essentially a Cholesky  decomposition, and we use it to ensure the positive definiteness of the posterior covariance matrix. We then derive the EBLO 
\begin{align}
	&\Lcal = \EE_{q(\b, \f)} \left[\log\frac{p(\text{Joint})}{q(\b, \f)}\right] \notag \\
	&=\EE_q \left[\log \frac{\text{OtherTerms} \cdot p(\b)\cancel{p(\f|\b)} p(\Scal|\f)}{q(\b)\cancel{p(\f|\b)}}\right]. \notag 
\end{align}
Now we can see that the full conditional Gaussian distributions $p(\f|\b)$ is canceled. We only need to calculate the $h \times h$ covariance matrix for $p(\b)$, which is very small. Hence, the cost  is largely reduced. The detailed ELBO is given by
\begin{align}
	&\Lcal =  - \text{KL}(q(\b)\| p(\b)) + \sum\nolimits_{n=1}^N \log w_{\bi_n} \notag  \\
	&  +\sum\nolimits_{k=1}^K \sum\nolimits_{j=1}^{D_k}\sum\nolimits_{r=1}^R \log \text{Beta}(v^k_{rj}|1, \alpha)  \label{eq:elbo}\notag \\ 
	&- \sum_{n=1}^N \EE_q\EE_{p(t)}[ T\left(f(\x_{\bi_n},t)\right)^2 ] \notag \\
	&+ \sum_{n=1}^N\sum_{j=1}^{m_{\bi_n}} \EE_q[\log (f(\x_{\bi_n}, s_{\bi_n, j}))^2] , \notag
\end{align}
where $p(t) = \text{Uniform}(0, T)$, and $\text{KL}(\cdot, \cdot)$ is the Kullback-Leibler divergence. \cmt{Note that we use the fact that $\int_0^T \left(f(\x_{\bi_n},t)\right)^2\d t = \EE_{p(t)}\left[T\cdot \left(f(\x_{\bi_n},t)\right)^2\right] $.}We maximize $\Lcal$ to estimate the variational posterior $q(\b)$ and the other parameters, including the stick-breaking variables $\{\v^k_j\}$, kernel parameters, \etc Due to the additive structure over both the interactions and their events, it is straightforward to combine with the reparameterization trick~\citep{kingma2013auto} to perform  efficient stochastic mini-batch optimization. 

However, since our embeddings are constructed from the logarithm of the sociabilities (in $[0, 1]$), $\u^k_j = [\log(\omega^k_{1j}); \ldots; \log(\omega^k_{Rj})]$, and  these sociabilities are often small, close to zero, their log scale can be quite big, \eg hundreds. As a result, when we feed the input $\x_{\bi} = [\u^1_{i_1}, \ldots, \u^K_{i_K}]^\top$ to the GP kernel (\eg we used SE kernel in the experiments), it is easy to incur numerical issues or make the kernel matrix stuck to be diagonal. To address this issue, we use the batch normalization method~\citep{ioffe2015batch}. That is, we jointly estimate an (empirical) mean and standard deviation for each embedding element during our stochastic mini-batch optimization. Denote them by $\boldeta$ and $\bsigma$. Each time,  we first normalize each $\x_{\bi}$ by $$\x_{\bi} \leftarrow \frac{\x_{\bi} - \boldeta}{\bsigma},$$ and then feed them to the kernel; $\boldeta$ and $\bsigma$ are jointly updated with all the other parameters using stochastic gradients. We empirically found the numerical problem disappears, and the learning is effective (see Sec. \ref{sect:exp}). 

\noindent\textbf{Algorithm Complexity.} The time complexity of our inference algorithm is $\Ocal\big( m h^2 +KR\sum_{k=1}^KD_k\big)$ where $m = \sum_{n=1}^N m_{\bi_n}$ is the total number of events. Since $h \ll m$,  the computational cost is linear in $m$. The space complexity is $\Ocal (h^2 + R\sum_{k=1}^KD_k )$, including the storage of the prior and posterior covariance matrices for pseudo outputs $\b$ and  embeddings $\Ucal$.

\section{Related Work}
It is natural to represent high-order interactions by multidimensional arrays or tensors. Tensor factorization is the fundamental framework for tensor analysis. Classical tensor factorization approaches  include CP~\citep{Harshman70parafac}  and Tucker~\citep{Tucker66} decomposition, based on which numerous other methods have been proposed:~\citep{Chu09ptucker,kang2012gigatensor,YangDunson13Tensor,choi2014dfacto,du2018probabilistic,fang2021bayesian}, to name a  few. Recently, nonparametric and/or neural network  factorization models~\citep{zhe2015scalable,zhe2016distributed,zhe2016dintucker,liu2018neuralcp,pan2020streaming,tillinghast2020probabilistic,fang2021streaming,tillinghast2021nonparametric} were developed to estimate nonlinear relationships in data, and have shown advantages over popular multilinear methods in prediction accuracy. When dealing with temporal information, existing methods mainly use homogeneous Poisson processes and decompose the event counts~\citep{chi2012tensors,HaPlKo15,Hu2015CountTensor}. More advanced approaches further partition the time stamps into different steps, and perform count factorization across the time steps~\citep{xiong2010temporal, schein2015bayesian,Schein:2016:BPT:3045390.3045686,schein2019poisson}. Recently, \citet{zhe2018stochastic} used Hawkes processes to estimate the local triggering effects between the events, and modeled the triggering strength with a kernel of the embeddings of the interactions. \citet{pan2020scalable} modeled the time decay as another kernel of the embeddings, and developed scalable inference for long event sequences. \citet{wang2020self} proposed a non-Hawkes, non-Poisson process to estimate the triggering and inhibition effects between the events. All these are temporal point processes that focus on rate modeling, and are different from the PPPs (with mean measure) in \ours to sample sparse interaction structures. 
\citet{lloyd2015variational} proposed GP modulated Poisson processes and also used the square link to ensure a positive rate function.  However, the work is purely about event modeling and does not learn any embedding. With the SE kernel, it derives an analytical form of ELBO. However, since our model includes the embedding (log sociabilities) in the GP prior, the ELBO is analytically intractable, and we use the reparameterization trick to conduct stochastic optimization. Recently, \citet{pan2021self} proposed a self-adaptable point process for event modeling, which can estimate both the triggering and inhibition effects within the events. More important, they construct a GP based component to enable a nonparametric estimate of the time decays of these effects. Their point process is not a Poisson process any more. 

Our hyper-graph prior is inherited from the sparse tensor process in~\citep{tillinghast2021nonparametric}, which can be viewed as a multi-dimensional extension of the pioneer work of \citet{caron2014sparse,caron2017sparse}, who first  used completely random measures (CRMs)~\citep{kingman1967completely,kingman1992poisson,lijoi2010models}, such as Gamma processes (\gaps)~\citep{hougaard1986survival}, to generate sparse random graphs. \cmt{\citet{williamson2016nonparametric} considered the case of  Gamma processes~\citep{hougaard1986survival,brix1999generalized}, and used the normalized version, \ie Dirichlet processes~\citep{ferguson1973bayesian}, to develop models for link prediction.} However, these prior works only show the asymptotic sparse guarantee (\ie the sparsity ratio converges to zero at the limit), yet not giving any convergence rate estimate for popular \gaps that are convenient for  inference and computation.  Our work fills this gap by giving strong asymptotic bounds about the sparsity, including both a lower and upper bound, which can reveal more refined insight about these sparse priors. Furthermore, we couple the sparse  prior with matrix GPs  to jointly sample the interactions (\ie hyperedges) and their event sequences. In this way, the embeddings can assimilate both the sparse structural information underlying the present interactions and the hidden temporal relationship of the participants. 

\cmt{
In spite of the success of the existing models for interaction events, they overlook the sparse nature of the present interactions (against all possible interactions between the participants), and hence miss important structural information. To overcome this limitation, we view the participants of the interactions events as the nodes in a hypergraph, and their interactions as hyperedges. The event sequences are attached on the edges. We propose a stochastic process to sample the hypergraphs, with the guarantee of asymptotic sparsity. In the modeling of ordinary graphs,  \citet{caron2014sparse,caron2017sparse} first  used completely random measures (CRMs)~\citep{kingman1967completely,kingman1992poisson,lijoi2010models} to generate sparse random graphs. \citet{williamson2016nonparametric} considered the case of  Gamma processes~\citep{hougaard1986survival,brix1999generalized}, and used the normalized version, \ie Dirichlet processes~\citep{ferguson1973bayesian}, to develop models for link prediction. Our model can be viewed as an extension of these seminal works  for hypergraphs. Furthermore, we couple our stochastic hypergraph process with matrix Gaussian process  to jointly sample the interactions (\ie hyperedges) and their event sequences. In this way, the embeddings can assimilate both the sparse structural information underlying the present interactions and the hidden temporal relationships between the participant nodes. For sparse random graph generation,  \citet{crane2015framework,cai2016edge}  also proposed edge-exchange models; see thorough discussions in \citep{crane2018edge}. 
}
\section{Experiment}\label{sect:exp}
\begin{figure}
	\centering
	\includegraphics[width=0.4\textwidth]{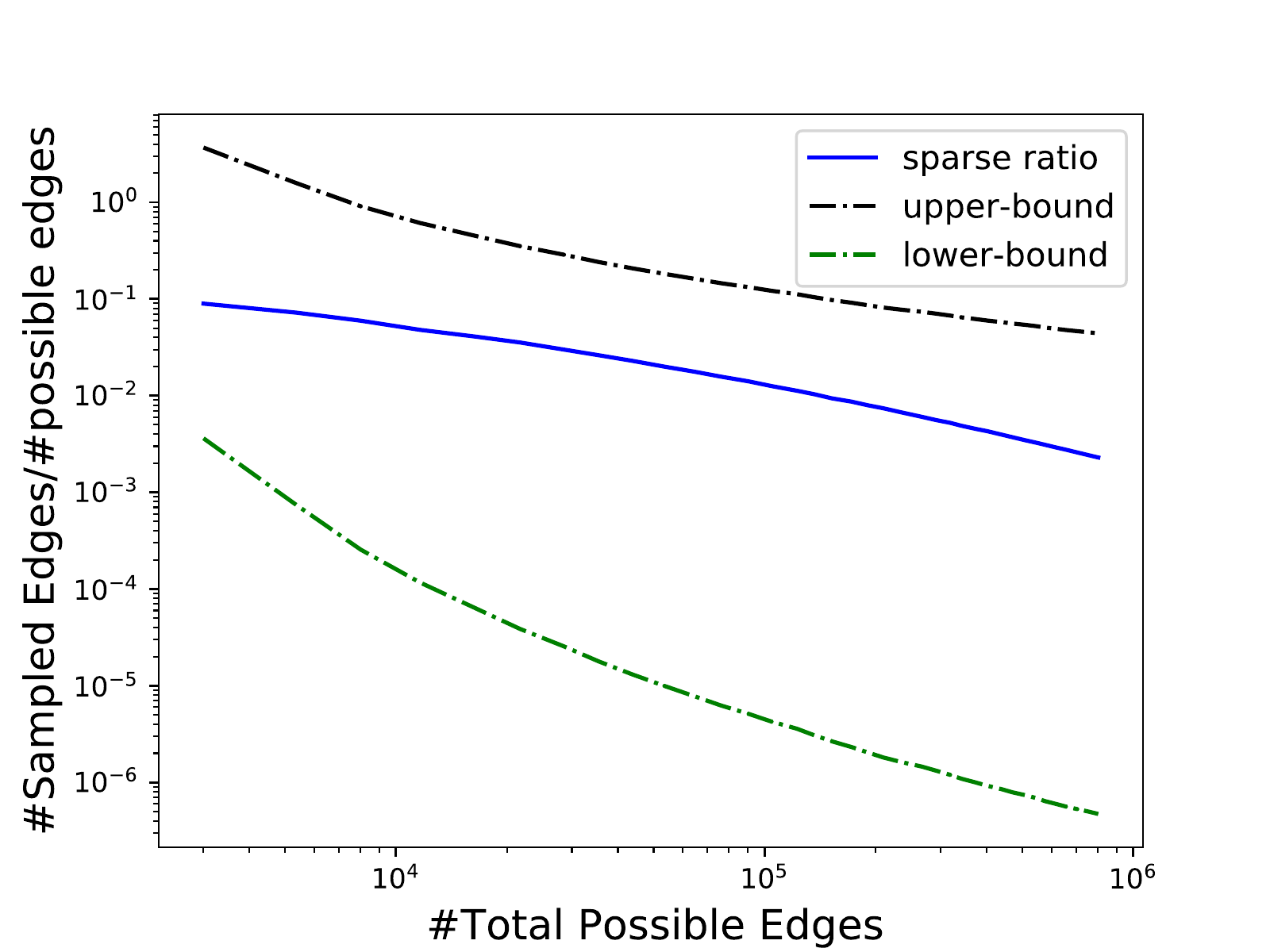}
	\caption{\small Sparsity ratio of the sampled hypergraphs and bounds.} \label{fig:sparse}
\end{figure}
\subsection{Sparsity Ratio Investigation}
\vspace{-0.05in}
We first examined if our theoretical bounds in Lemma \ref{lem:lem2} match the actual sparsity ratio in the sampled hyper-graphs. To this end, we followed~\citep{tillinghast2021nonparametric} to sample a series of hyper-graphs with three-way edges (interactions), namely $K=3$. We set $R=1$ and varied $\alpha$ in $[2, 20]$. For each particular $\alpha$, we independently sampled $200$ hypergraphs and computed average ratio of the sampled edges. We calculated the bounds accordingly. We show the results in a log-log plot as in Fig. \ref{fig:sparse}. As we can see, the bounds clamp the actual sparsity ratio and match the trend well. The upper bound is tighter. Hence, these bounds can provide a reasonable estimate of the convergence rate and characterize the asymptotic behaviors of the structural sparsity. 
\subsection{Predictive Performance}
\begin{figure*}[t]
	\centering
	\setlength\tabcolsep{0.01pt}
	\begin{tabular}[c]{ccc}
		\multicolumn{3}{c}{\includegraphics[width=0.618\textwidth]{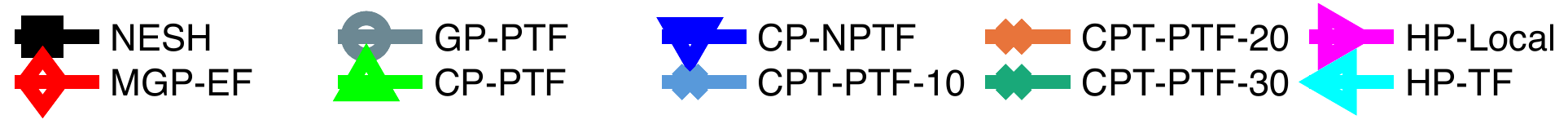}}
		\\
		\begin{subfigure}[t]{0.3\textwidth}
			\centering
			\includegraphics[width=\textwidth]{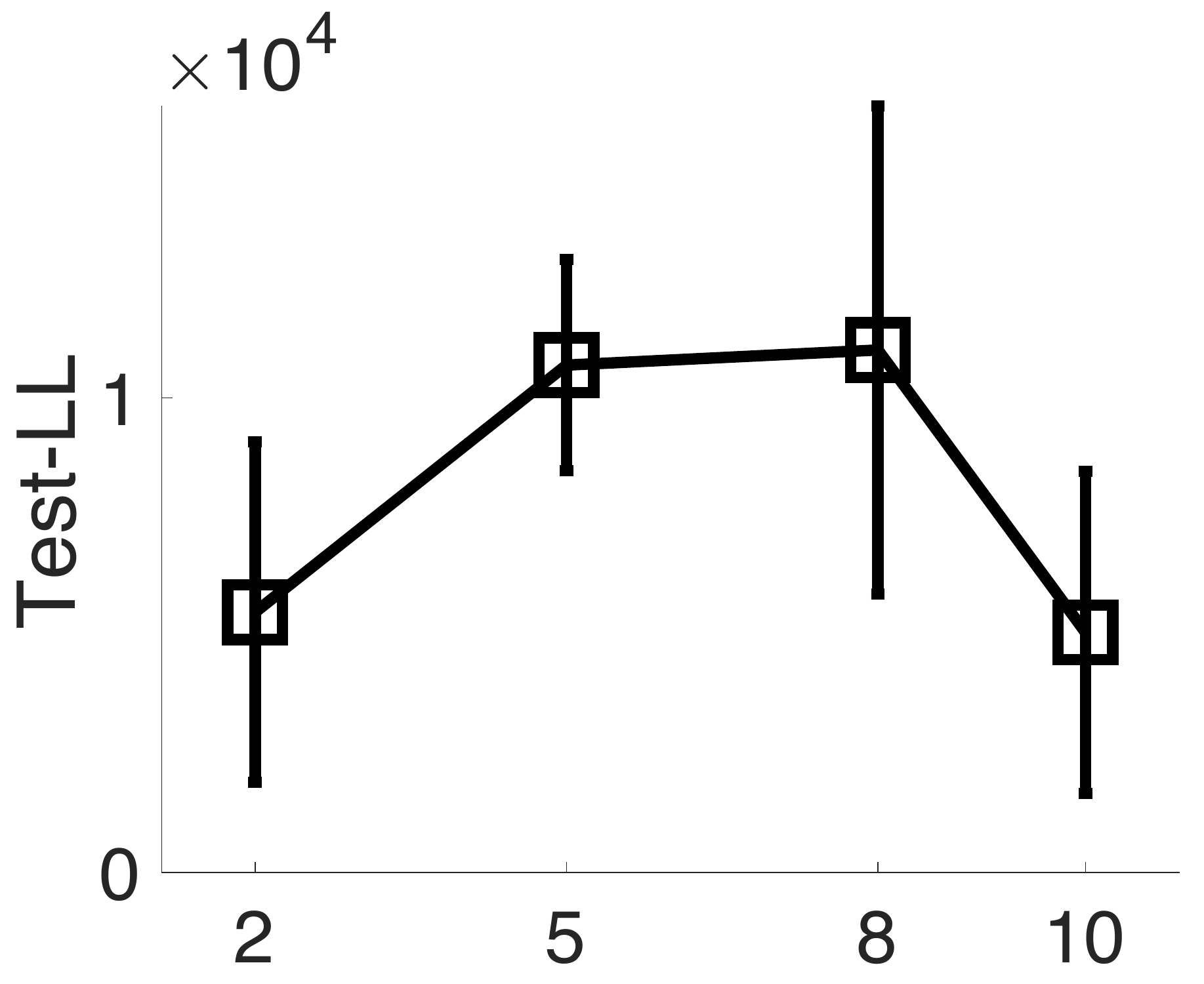}
		\end{subfigure} 
		&
		\begin{subfigure}[t]{0.3\textwidth}
			\centering
			\includegraphics[width=\textwidth]{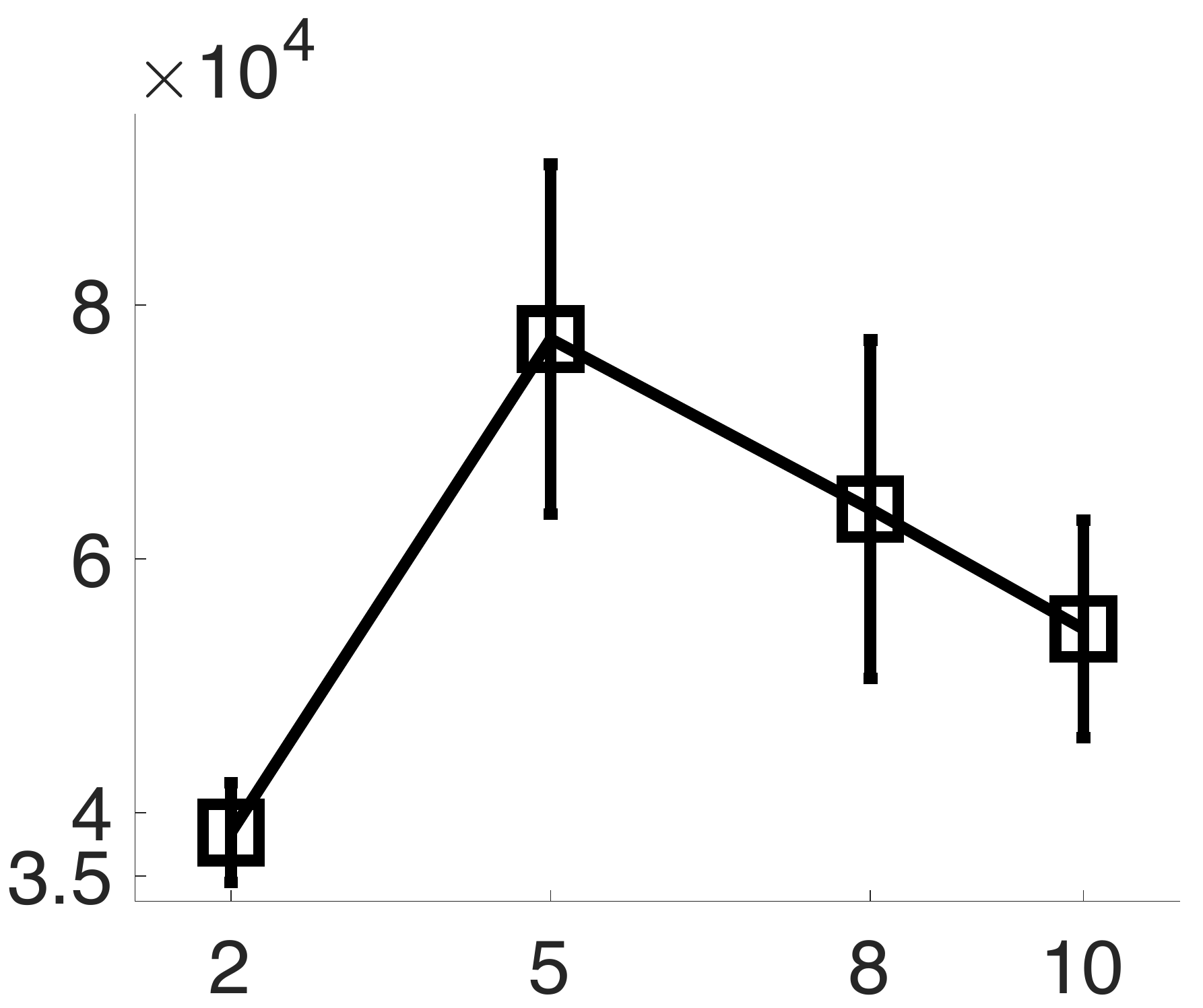}
		\end{subfigure} 
		&
		\begin{subfigure}[t]{0.3\textwidth}
			\centering
			\includegraphics[width=\textwidth]{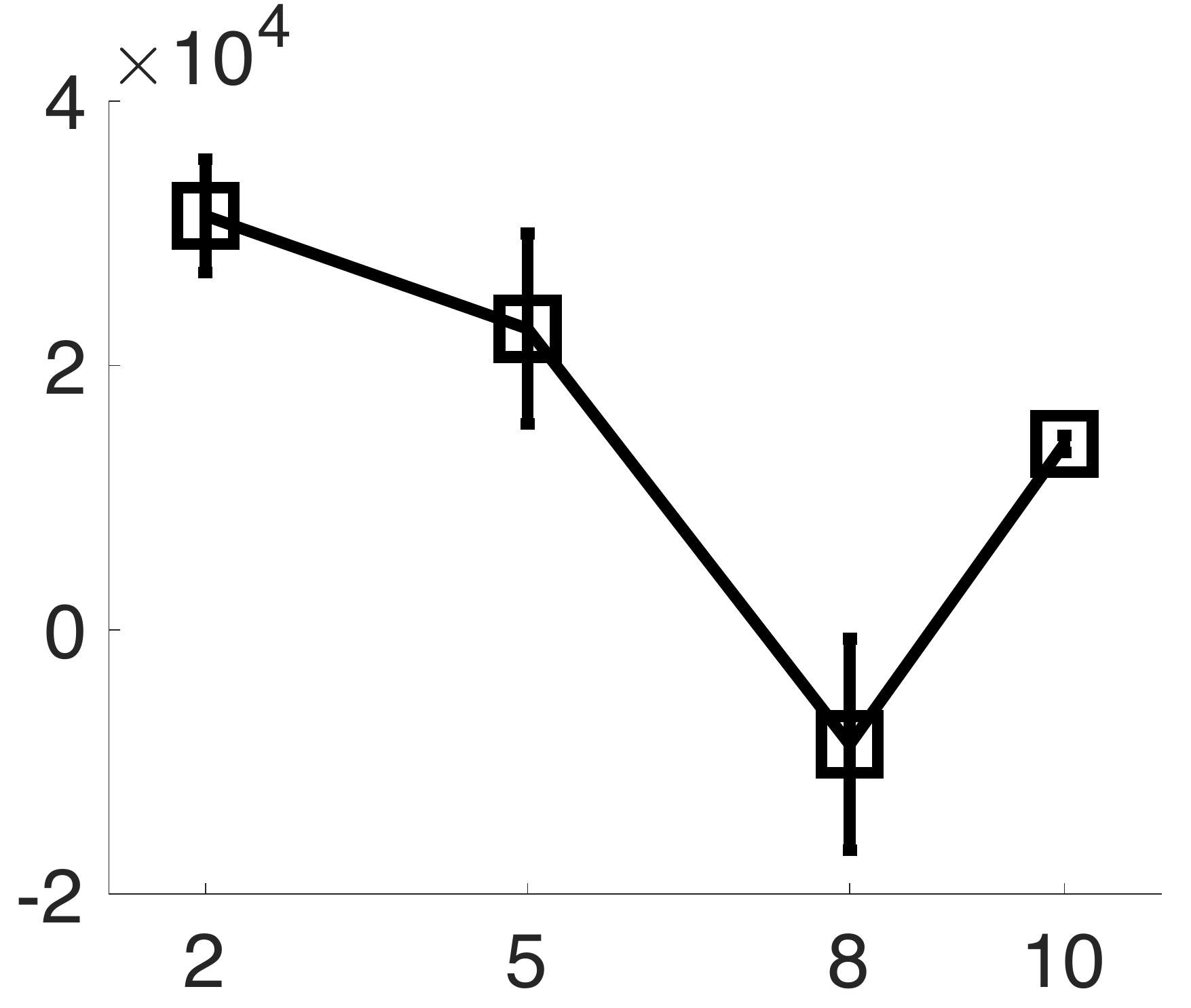}
		\end{subfigure} \\
		\begin{subfigure}[t]{0.3\textwidth}
			\centering
			\includegraphics[width=\textwidth]{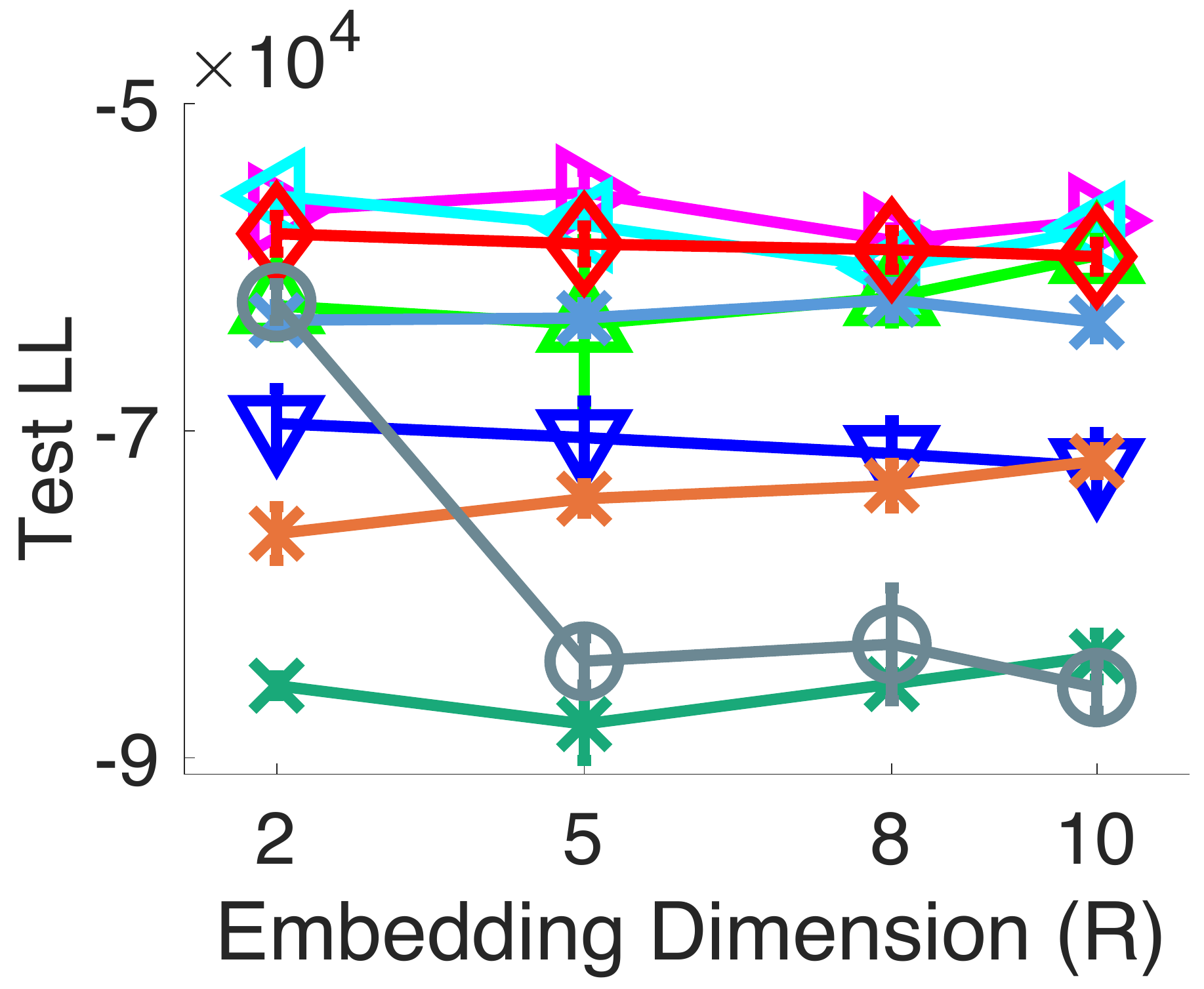}
			\caption{\textit{Crash}}
		\end{subfigure} 
		&
		\begin{subfigure}[t]{0.3\textwidth}
			\centering
			\includegraphics[width=\textwidth]{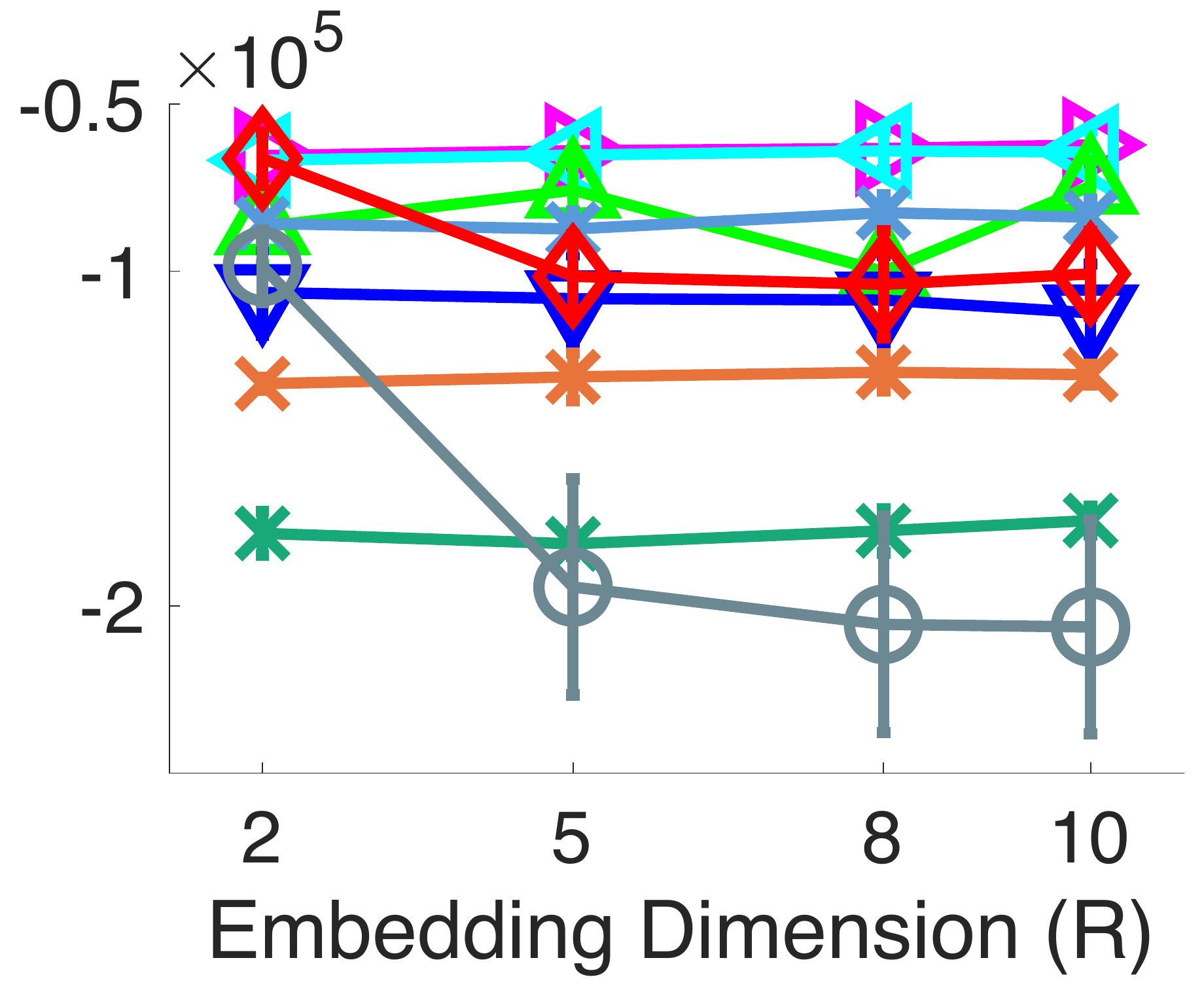}
			\caption{\textit{Taobao}}
		\end{subfigure} 
		&
		\begin{subfigure}[t]{0.3\textwidth}
			\centering
			\includegraphics[width=\textwidth]{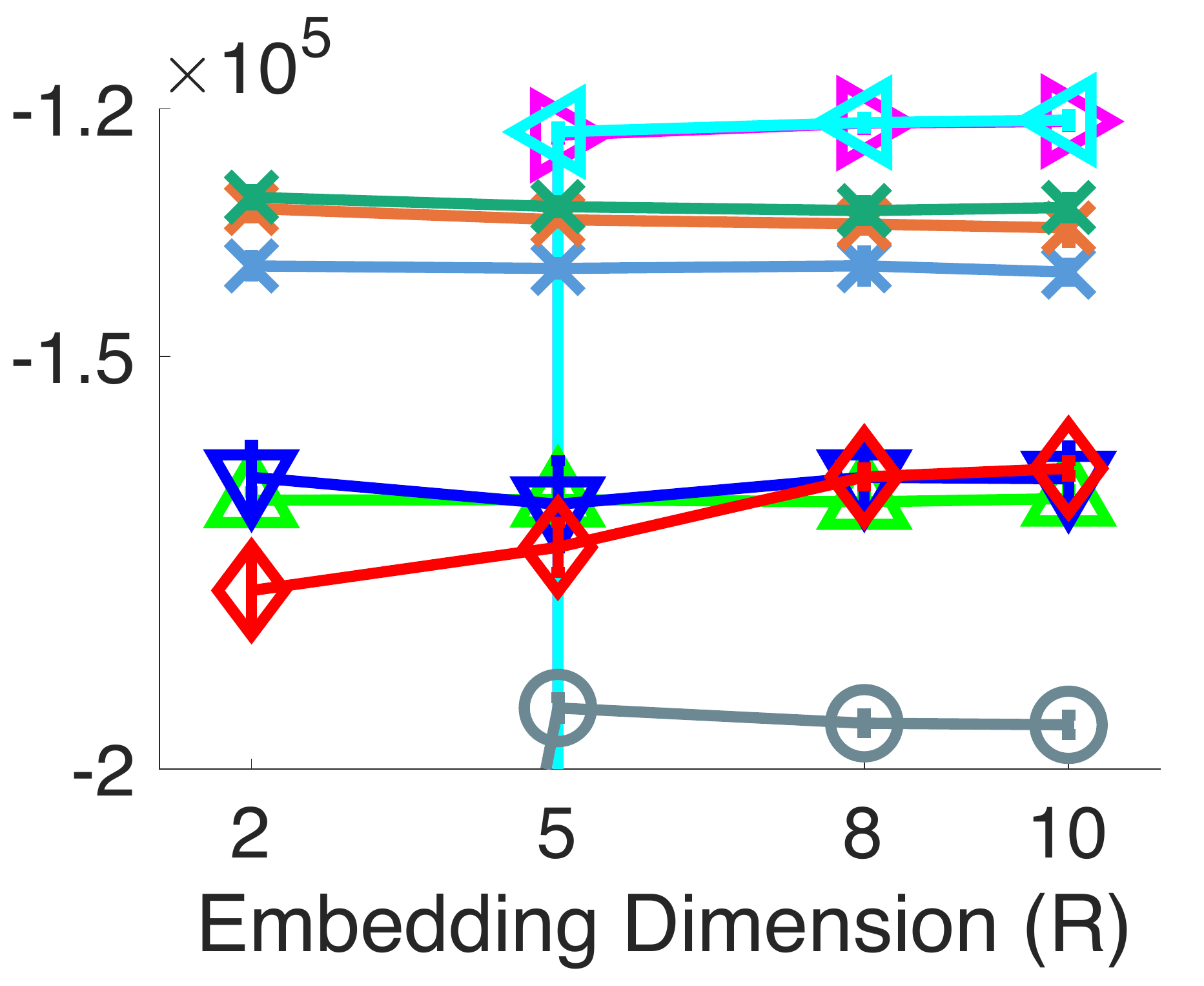}
			\caption{\textit{Retail}}
		\end{subfigure}
	\end{tabular}

	\caption{Test log-likelihood (LL) on real-world datasets. CPT-PTF-\{10, 20, 30\} means running CPT-PTF with 10, 20 and 30 time steps. The results were averaged over five runs. } 	
	\label{fig:test-ll}
\end{figure*}
\begin{figure*}[t]
	\centering
	\setlength\tabcolsep{0.0pt}
	\begin{tabular}[c]{cccc}
		\begin{subfigure}[t]{0.2\textwidth}
			\includegraphics[width=\textwidth]{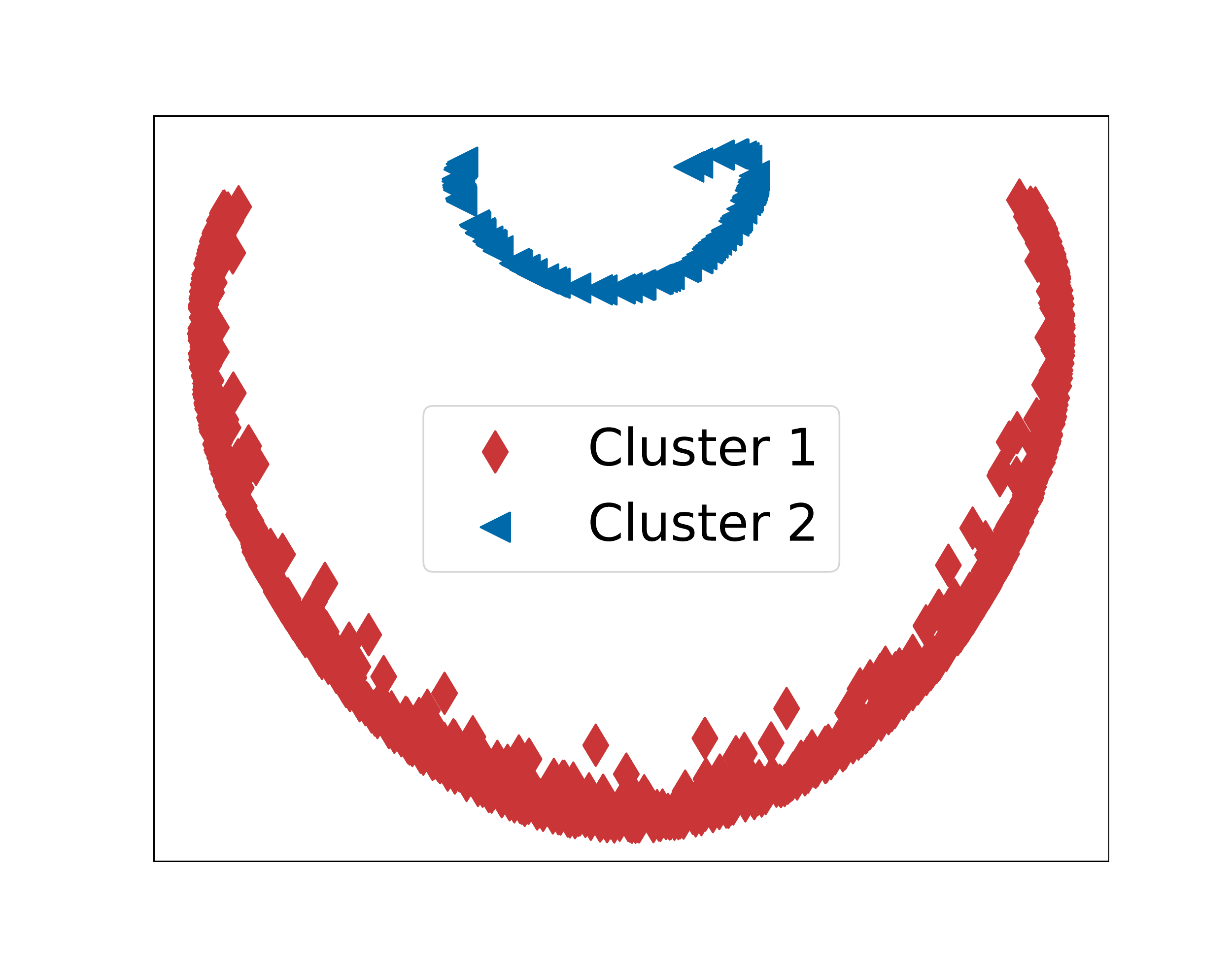}
			\caption{\textit{Users}}
		\end{subfigure}
		&
		\begin{subfigure}[t]{0.2\textwidth}
			\includegraphics[width=\textwidth]{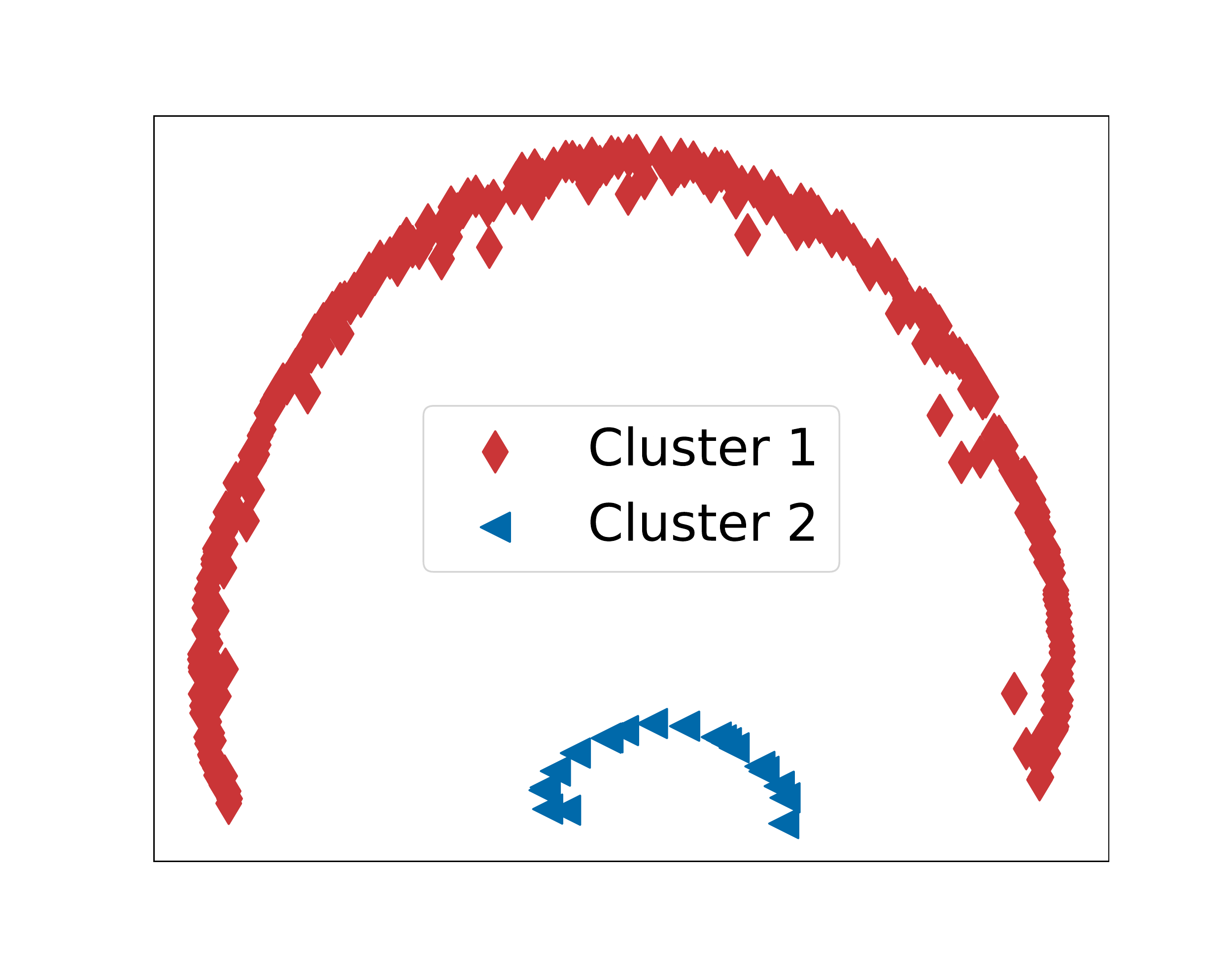}
			\caption{\textit{Sellers}}
		\end{subfigure}
		&
		\begin{subfigure}[t]{0.2\textwidth}
			\includegraphics[width=\textwidth]{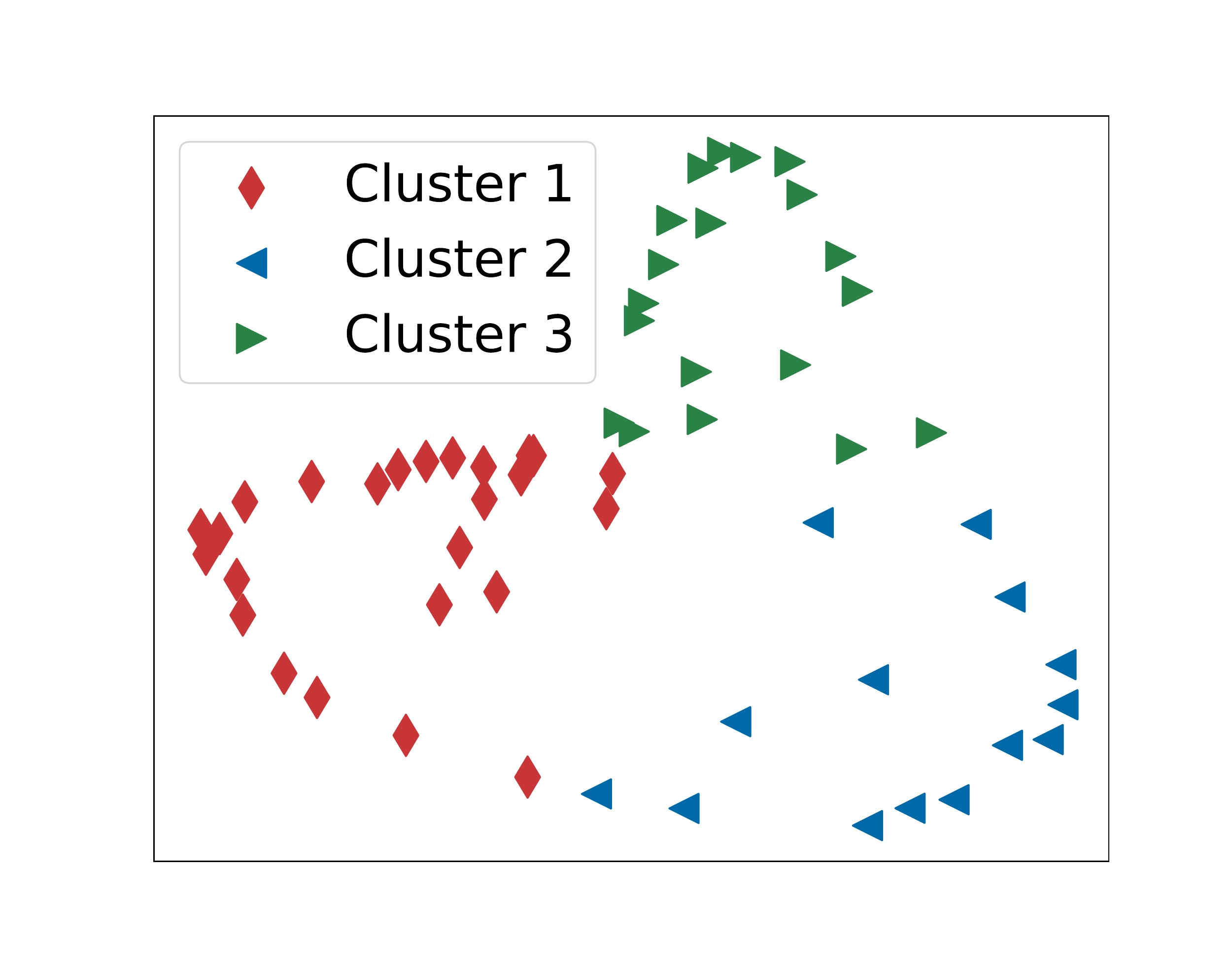}
			\caption{\textit{Item Categories}}
		\end{subfigure}
		&
		\begin{subfigure}[t]{0.2\textwidth}
			\includegraphics[width=\textwidth]{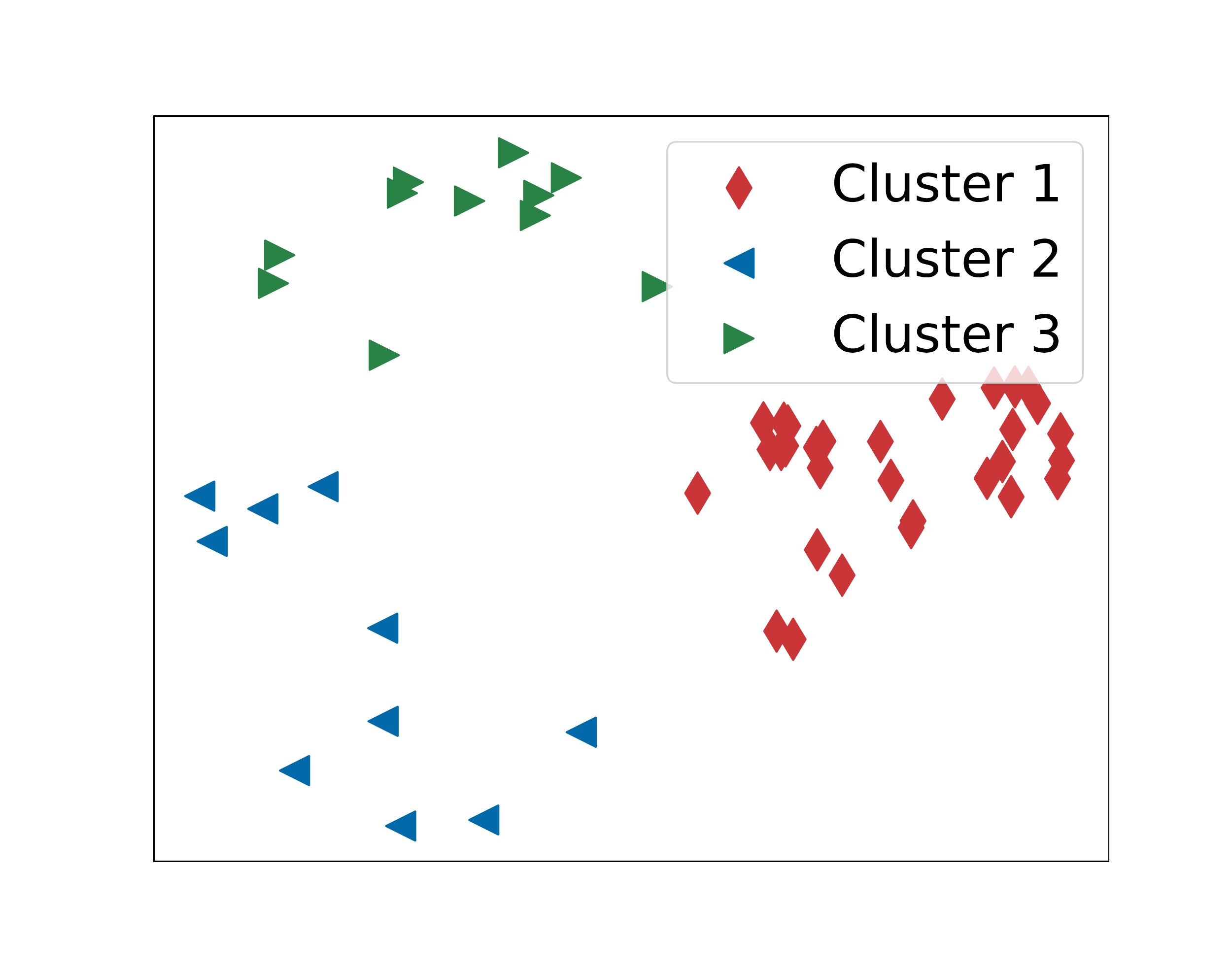}
			\caption{\textit{States}}
		\end{subfigure}
	\end{tabular}
	\caption{\small Structures of the estimated embeddings on \textit{Taobao} (a, b, c), \textit{Crash} (d). The points represent the participant nodes, and the colors indicate their cluster memberships.}\label{fig:structure}
\end{figure*}

\noindent\textbf{Datasets}. We then examined the predictive performance of \ours on the following real-world datasets. (1) \textit{Taobao} (\url{https://tianchi.aliyun.com/dataset/dataDetail?dataId=53}), the shopping events in the largest online retail platform of China, from 07/01/2015 to 11/30/2015, which are interactions between $980$ \textit{users}, $274$ \textit{sellers}, $631$ \textit{items},  $58$ \textit{categories} and $2$ \textit{options}. There are in total $16,609$ distinct interactions and $69,833$ events. \cmt{The existent (distinct) interactions take $0.000085\%$ of all possible interactions.} (2) \textit{Crash} (\url{https://www.kaggle.com/usdot/nhtsa-traffic-fatalities}), fatal traffic crashes in US 2015, within $51$ \textit{states}, $288$ \textit{counties}, $2,098$ cities and $5$ \textit{landuse-types}. There are $8,691$ distinct interactions \cmt{($0.0056\%$ of all possible interactions) }and $32,052$ events in total. (3) \textit{Retail} (\url{https://tianchi.aliyun.com/dataset/dataDetail?dataId=37260}), online retail records from \texttt{tmall.com}. It includes interaction events among \textit{stock items} and \textit{customers}. We have $3,310$ and $1,000$ unique items and customers, among which are $31,122$ distinct interactions and $70,000$ events in total. We can see that all these datasets are very sparse. The existent interactions take $0.000085\%$, $0.0056\%$ and $0.94\%$ on \textit{Taobao}, \textit{Crash} and \textit{Retail}, respectively. 

\noindent{\textbf{Competing Methods.}} We compared with the following popular and/or state-of-the-art tensor decomposition methods that deal with interaction events.
(1) CP-PTF, similar to ~\citep{chi2012tensors}, the homogeneous Poisson process (PP) tensor decomposition, which uses CP to factorize the event rate for each particular interaction and the square link to ensure the positiveness (consistent with \ours). (2) CPT-PTF, similar to ~\citep{schein2015bayesian}, which extends CP-PTF by introducing time steps in the tensor. The embeddings of the time steps are assinged a conditional Gaussian prior~\citep{xiong2010temporal} to model their dynamics. (3) GP-PTF, which uses GPs to estimate the square root of of the rate for each particular interaction as a nonlinear function of the associated embeddings. (4) CP-NPTF, non-homogeneous Poisson process tensor factorization where the event rate is modelled as a parametric form,  $\lambda_{\bi}(t) =t\cdot  \big(\mathrm{CP(\bi)}\big)^2 $. Here $\mathrm{CP}(\bi)$ is CP decomposition of the entry (interaction) $\bi$.  (5) HP-Local~\citep{zhe2018stochastic}, Hawkes process based decomposition that uses a local time window to model the rate and to estimate the local excitation effects among the nearby events, (6) HP-TF~\citep{pan2020scalable}, another Hawkes process based on factorization method that models both the triggering strength and decay as kernel functions of the embeddings. Both HP-Local and HP-TF use a GP to model the base rate as a function of embeddings. In addition, we compared with (7) MGP-EF, matrix GP based events factorization. It is the same as our method in applying a matrix GP prior over the rates of distinct interactions. However, MGP-EF places a standard Gaussian prior over the embeddings, and so does not model the structure sparsity.

\noindent\textbf{Settings.} We implemented \ours, HP-Local, HP-TF and MGP-EF with Pytorch~\citep{paszke2019pytorch}, and the other methods with MATLAB. For all the approaches employing GPs, we used the same variational approximation as in \ours (our method), and set the number of pseudo inputs to $100$. We used the square exponential (SE) kernel and initialized the kernel parameters with $1$.  For HP-Local, the local window size was set to $50$. For our method, we chose $\alpha$ from $\{0.5, 1.0, 1.5, 2.5, 3\}$. We conducted stochastic mini-batch optimization for all the methods, where the batch size was set to $100$. We used ADAM~\citep{kingma2014adam} algorithm, and the learning rate was tuned from $\{5\times 10^{-4}, 10^{-3}, 3\times 10^{-3}, 5 \times 10^{-3}, 10^{-2}\}$. We ran each method for $400$ epochs, which is enough to converge. We randomly split each dataset into $80\%$ sequences for training, and the remaining $20\%$ for test. We varied $R$, the dimension of the embeddings, from $\{2, 5, 8, 10\}$. For CPT-PTF, we tested the number of time steps from $\{10, 20, 30\}$.  We ran the experiments for five times, and report the average test log-likelihood and its standard deviation in Fig. \ref{fig:test-ll}. 

\textbf{Results.} As we can see from Fig. \ref{fig:test-ll}, \ours consistently outperforms all the competing methods by a large margin. Since the test log-likelihood of \ours is much larger than that of the other methods, we report the result of \ours separately (\ie in the top figures) so that we can compare the difference between the competing methods. It can be seen that MGP-EF is much better than GP-PTF, implying that introducing a time kernel to model non-homogeneous Poisson process rates is more advantageous. In addition, MGP-EF is much better than or comparable to CP-NPTF (see Fig. \ref{fig:test-ll}a). Since both methods uses non-homogeneous Poisson processes, the result demonstrates the advantage of nonparametric rate modeling (the former) over the parametric one (the latter). In most cases, HP-Local and HP-TF shows better or comparable prediction accuracy than MGP-EF. This is reasonable, because the two methods use Hawkes processes that can capture more refined temporal dependencies, \ie excitation effects between the events. However, both HP-Local and HP-TF cannot capture the sparse structure within the present interactions, and their performance is still much inferior to \ours. Finally, GP-PTF, HP-Local and HP-TF broke down at $R=2$ on \textit{Retail} dataset. We found their learning was unstable. In some splits, their predictive likelihood are very small, leading to much worse average likelihood than the other methods. Note that they did not use batch normalization as in \ours and MGP-EF, which might cause the learning instability under some settings.

\subsection{Pattern Discovery}
Next, we examined if \ours can discover hidden patterns from data.  To this end, we set $R$ to $5$, and ran \ours  on \textit{Taobao}  and \textit{Crash}. Then we applied kernel PCA~\citep{scholkopf1998nonlinear} with the SE kernel to project the embeddings onto a plane. We then ran clustering algorithms to find potential structures. As we can see from Fig. \ref{fig:structure} a and b, the embeddings of \textit{users} and \textit{items} from \textit{Taobao} dataset exhibit interesting and clear cluster structures, which might correspond to separate interests/shopping habits. Note that we ran DBSCAN~\citep{ester1996density} rather than k-means to obtain the clusters. In addition, the embeddings of item \textit{categories} on \textit{Taobao} also shows a clear structure that was discovered by k-means (see Fig. \ref{fig:structure}c).  Although  the \textit{Taobao} dataset have been completely anynoymized and we cannot investigate the meaning of these clusters, potentially they can be  useful for tasks such as marketing~\citep{zhang2017nonlinear}, recommendation~\citep{liu2015non,tran2018regularizing} and click-through-rate prediction~\citep{pan2019warm}.  In addition, the embeddings of the \textit{states} from \textit{Crash} dataset also exhibit clear structures (see Fig. \ref{fig:structure}d). We have checked the geolocation of these states and found the states grouped together are often neighborhoods. This is reasonable in that neighboring states might bear a resemblance to each other, in traffic regulations (\eg speed limit), road conditions, driving customs, weather changes. \etc All these might lead to similar or closely related  patterns of traffic accident rates.

\section{Conclusion}
We have presented \ours, a novel nonparametric embedding method for sparse high-order interaction events. Not only can our method estimate the complex temporal relationships between the participants, our model  is also able to capture the structural  information underlying the observed sparse interactions. Our theoretical bounds enable convergence rate estimate and reveal insights about the asymptotic behaviors of the sparse prior over hypergraphs or tensors.    In the future,  we will extend our model to more expressive point processes, such as Hawkes processes, and discover more refined  temporal patterns.

\section*{Acknowledgments}
This work has been supported by NSF IIS-1910983, NSF DMS-1848508 and NSF CAREER Award IIS-2046295.

\bibliographystyle{apalike}
\bibliography{SparseIE}

\newpage
\appendix
\onecolumn
\section*{Appendix}

%
%
%
%
%
%
%
%
%
%
%
%
%
\newtheorem{Th}{Theorem}
\newtheorem{Lemma}[Th]{Lemma}
\newtheorem{Cor}[Th]{Corollary}
\newtheorem{Prop}[Th]{Proposition}
 \theoremstyle{definition}
\newtheorem{Def}[Th]{Definition}
\newtheorem{Conj}[Th]{Conjecture}
\newtheorem{Rem}[Th]{Remark}
\newtheorem{?}[Th]{Problem}
\newtheorem{Ex}[Th]{Example}
\newcommand{\im}{\operatorname{im}}
\newcommand{\Hom}{{\rm{Hom}}}
\newcommand{\diam}{{\rm{diam}}}
\newcommand{\ovl}{\overline}
\newcommand\independent{\protect\mathpalette{\protect\independenT}{\perp}}
\def\independenT#1#2{\mathrel{\rlap{$#1#2$}\mkern2mu{#1#2}}}
\def\R{{\mathbb R}}
\def\Q{{\mathbb Q}}
\def\Z{{\mathbb Z}}
\def\N{{\mathbb N}}
\def\C{{\mathbb C}}
\def\E{{\mathbb E}}
\def\R{{\mathbb R}}
\def\Y{{\mathcal Y}}
\def\L{{\mathcal L}}
\def\H{{\mathcal H}}
\def\D{{\mathcal D}}
\def\P{{\mathbb P}}
\def\M{{\mathbb M}}
\def\V{{\mathcal V}}
\def\S{{\mathbb S}}
\def\A{{\mathbf A}}
\def\x{{\mathbf x}}
\def\b{{\mathbf b}}
\def\a{{\mathbf a}}
\def\Ph{{\mathbf {\Phi}}}

\def\h{{\mathbf{h}}}
\def\G{{\Gamma}}
\def\s{{\sigma}}
\def\e{{\varepsilon}}
\def\l{{\lambda}}
\def\p{{\phi}}
\def\v{{\mathbf{v}}}
\def\t{{\theta}}
\def\z{{\zeta}}
\def\o{{\omega}}
\def\y{{\mathbf{y}}}
\def\g{{\mathbf{g}}}
\def\u{{\mathbf{u}}}
\def\w{{\mathbf{w}}}
In this section, we provide detailed proof for Lemma 3.2. 
To make the ideas accessible to a broad audience, we also give a brief introduction to the L\'evy-Khintchine formula as well as the intuition behind it. 
Our introduction includes the Gamma Processes (\gaps), which are used to sample the intensity measures for a sequence of Poisson Point Processes (PPPs) that are used for sparse tensor/hypergraph construction, as a special case. 
A comprehensive treatment of the related topics can be found in, for instance, \cite{ken1999levy}.
 
\section{The L\'evy-Khintchine formula}\label{sec1}

The sparse tensor model introduced in \cite{tillinghast2021nonparametric} first generates a discrete measure using \gaps, which are a special type of L\'evy process. 
To better understand the process, we take a brief detour to L\'evy processes. 
 
A L\'evy process $\{X_t\}_{t\geq 0}$ is an $\R^d$-valued random process such that
\begin{itemize}
\item $X_0 = 0$  a.s.;
\item $X_t$ has stationary and independent increments;
\item For every $t$, $X_t$ is right-continuous and has a well-defined left-limit.
\end{itemize}
$X_t$ is uniquely determined by $X_1$, which is an infinitely divisible random variable. 
Indeed, if the characteristic function (CF) of $X_1$ is $\phi(\xi)$, then the CF of $X_t$ is $\phi^t(\xi)$ for $t\geq 0$.
A complete understanding of $\phi(\xi)$ is sufficient to characterize $X_t$, and this can be done via the L\'evy-Khintchine formula:

\begin{Th}[L\'evy-Khintchine]\label{LK}
Let $X$ be an $\R^d$-valued random variable. $X$ is infinitely divisible if and only if the CF of $X$, $\phi_X(\xi): = \E[e^{i\xi\cdot X}]$, takes the form
\begin{align}
\phi_X(\xi) = \exp\left\{-\Psi(\xi)\right\},
\end{align}
where 
\begin{align}
\Psi(\xi) = i(a\cdot \xi) + \frac{1}{2}\|\sigma\xi\|^2 + \int_{\R^d}(1-e^{i\xi\cdot z}+i(\xi\cdot z)\mathbf{1}_{(0,1)}) m(dz),\label{1}
\end{align}
where $m$ is a Borel measure on $\R^d$ satisfying $m(\{0\}) = 0$ and $\int_{\R^d}(1\wedge \|x\|^2) m(dx)<\infty$. 
Here $\Psi$ is called the L\'evy exponent of $X$. 
\end{Th}

As a consequence of Theorem \ref{LK}, we conclude that the CF of every L\'evy process $X_t$ can be written as 
\begin{align*}
\phi_{t}(\xi) = \E[e^{i\xi\cdot X_t}] = \exp\left\{-t\Psi(\xi)\right\},
\end{align*}
where $\Psi$ is the L\'evy exponent of $X_1$. 

To comprehend the path structure of $X_t$ using \eqref{1}, we appeal to the following facts:
\begin{itemize}
\item Addition of a CF's exponents corresponds to addition of independent random variables (processes);
\item The CF of a drifted Brownian motion $W_t = -at + \sigma B_t$ is $\exp\{-t(i(a\cdot \xi) - \frac{1}{2}\|\sigma\xi\|^2)\}$;
\item The CF of a compound Poisson process with jump parameter $m$ (i.e. a distribution) and rate parameter $\lambda$ (i.e. a L\'evy process) is $\exp\{-t\int_{\R^d}\lambda (1-e^{i\xi\cdot z}) m(dz)\}$.
\item The CF of a compensated compound Poisson process with jump parameter $m$ (i.e. a distribution) and rate parameter $\lambda$ (i.e. a L\'evy process and a martingale) is $\exp\{-t\int_{\R^d}\lambda (1-e^{i\xi\cdot z}+i(\xi\cdot z)) m(dz)\}$.
\end{itemize}

Let $I_0 = \{z: \|z\|\geq 1\}$ and $I_k = \{z: \|z\|\in [2^{-k-1}, 2^{-k})\}$ for $k\geq 1$. 
Rewrite \eqref{1} as 
\begin{align}
t\Psi(\xi) &= t(i(a\cdot \xi) + \frac{1}{2}\|\sigma\xi\|^2) + t\int_{I_0}m(I_0)(1-e^{i\xi\cdot z}) \frac{m(dz)}{m(I_0)}\nonumber\\
&\ \ \ \  + \sum_{k=1}^\infty t\int_{I_k}m(I_k)(1-e^{i\xi\cdot z}+i(\xi\cdot z)) \frac{m(dz)}{m(I_k)}\label{lll},
\end{align}
where the right-hand side corresponds to three independent processes: a drifted Brownian motion, a compound Poisson process, and a series of independent compensated compound Poisson processes. 
Under the integrability condition on $m$, the last part can be shown to converge using the martingale theory. 
Moving the drift term (cumulation of the compensated terms; deterministic) in the last part of \eqref{lll} into the Brownian motion, we decompose $X_t$ into two independent processes, a drifted Brownian motion and a pure-jump process (with countably many jumps).
This construction is the celebrated \emph{L\'evy-It\^o construction}. 

To end this section, we give a useful interpretation of compound Poisson processes. 
A compound Poisson process $Z_t$ with jump parameter $m$ and rate parameter $\lambda$ is defined as
\begin{align*}
&Z_t = \sum_{i=1}^{N_t}M_i,
\end{align*}
where $M_i\stackrel{iid}{\sim}m(dx)$ are independent of $N_t\sim\text{Poisson}(\lambda)$.
Note $\{M_i\}_{i\in [N_t]}$ follows a PPP with intensity measure $m(dx)$, for fixed $t$, we may also consider $Z_t$ as the integration of $x$ against the Poisson random measure on $[0,t]\times \R^d$. 

\section{Gamma processes}\label{sec:gp}

A Gamma process $X_t$ (with shape parameter $\beta$ and rate parameter $\lambda$) is a L\'evy Process with 
\begin{align*}
&X_t \sim \Gamma (\beta t, \lambda) & t>0.
\end{align*}
It can be checked that 
\begin{align*}
\E[e^{i\xi X_t}] = \left(1-\frac{i\xi}{\lambda}\right)^{\beta t} = \exp\left\{-t\int_\R(1-e^{-i\xi x})\frac{\beta e^{-\lambda x}}{x}\mathbf 1_{(0, \infty)}dx\right\},
\end{align*}
i.e., the L\'evy measure $m$ is 
\begin{align*}
m(dx) = \frac{\beta e^{-\lambda x}}{x}\mathbf 1_{(0, \infty)}dx.
\end{align*}
For convenience, we set $\beta = \lambda = 1$. 
From $m$ one can deduce the following properties of $X_t$:
\begin{itemize}
\item By the L\'evy-It\^o construction, $X_t$ is a pure-jump process (i.e. no Brownian part), i.e., $X_t$ has countably infinitely many jumps in $(0, t)$;
\item We can associate a sample path of $X_t$ (up to time $t$) with the following measure: 
\begin{align}
&w_t = \sum_{0<s\leq t}\Delta X_s\delta_{s}&\Delta X_s = X_s-\lim_{r\to s^-}X_{r},\label{gamma-measure}
\end{align}
where the summation is well-defined since there are at most countably many $s$ that $\Delta_s>0$.
$w_t$ is finite thanks to 
\begin{align*}
\int_{\R}(1\wedge x) m(dx)<\infty;
\end{align*}
\item The jump sizes of $X_t$, $\{\Delta X_s\}_{s\in\text{supp}(w_t)}\subset (0,\infty)$, follows a PPP with intensity measure $m$; see the last paragraph in Section \ref{sec1}.
\end{itemize}

\section{Sparse tensor/hypergraph processes}

The sparse hypergraph model considered in this paper is a superposition of $R$ independent sparse tensor process introduced in \cite{tillinghast2021nonparametric}. 
Without loss of generality, we assume $R = 1$; the general case can be analyzed similarly.
In this case, the sampled entries in the sparse tensor model are obtained as follows: Given $K\geq 2$ and time $\alpha$, we
\begin{itemize}
\item Use $K$ i.i.d. Gamma processes, $X^{(1)}_\alpha, \cdots, X^{(K)}_\alpha$, to generate discrete measures $W_1^{\alpha}, \cdots, W_K^{\alpha}$ as in \eqref{gamma-measure}. 
Here we change the notations to be consistent with the ones in the manuscript, with the index $r$ omitted;
\item Construct a product measure $M = \prod_{k=1}^KW_k^{\alpha}$ on $[0,\alpha]^K$;
\item Take $M$ as the intensity measure to construct a Poisson random measure $T$. In particular, one can take
\begin{align}
&T = \{Y_i\}_{i=1}^{|T|}\stackrel{\text{i.i.d.}}{\sim}\frac{M}{M([0, \alpha]^K)}&  |T|\sim\text{Poisson}(M([0, \alpha]^K))\independent Y_i,\label{myT}
\end{align}
where $X\independent Y$ means that $X$ and $Y$ are independent. 
The support of $T$ corresponds to sampled entries in the sparse tensor, and the marginals of the support of $T$ corresponds to the size of the sparse tensor in the respective dimension. 
\end{itemize}

To get an explicit rate of convergence of sparsity as $\alpha\to\infty$, we need to estimate the cardinality of the support of $T$, $N^\alpha$, as well as of the corresponding marginals, which are denoted by $D_1^\alpha, \cdots, D_K^\alpha$. 

\begin{Lemma}
Fix $K\geq 2$.
For a sparse tensor process defined as above, for all sufficiently large $\alpha$, there exists an absolute constant $C>0$ such that, with probability at least $1-(C\alpha)^{-K}$,   
\begin{align*}
\frac{e^{-1.03(2K)^{1/K}K(\log \alpha)^{1/K}}}{2K\log \alpha}\cdot\left[\frac{1.82}{(K-1)\log (1.01 \alpha)}\right]^K\leq \frac{N^\alpha}{\prod_{k=1}^{K}D_k^\alpha}\leq \left[\frac{2.11}{(K-1)\log (0.99 \alpha)}\right]^K.
\end{align*}
\end{Lemma}

\begin{proof}
\textbf{Analysis of $N^\alpha$.}
We begin by deriving an upper bound for $N^\alpha$.
Note $N^\alpha\leq |T|$, where the $|T|$ is defined in \eqref{myT}. 
Conditioned on $M$, 
\begin{align*}
&|T|\sim\text{Poisson}(\gamma_\alpha)&\gamma_\alpha = M([0,\alpha]^K) = \prod_{k=1}^KW_k^{\alpha}([0,\alpha]).
\end{align*}
By a Poisson tail estimate \citep[Exercise 2.3.6]{vershynin2018high}, we have
\begin{align}
\P\left(0.99\gamma_\alpha\leq |T|\leq 1.01\gamma_\alpha\right)\geq 1-2e^{-c_1\gamma_\alpha},\label{001}
\end{align}
where $c_1>0$ is an absolute constant. 
For each $k\in [K]$, 
\begin{align*}
W_k^{\alpha}([0,\alpha]) \stackrel{\eqref{gamma-measure}}{=} \sum_{0<s\leq \alpha}\Delta X^{(k)}_s = X^{(k)}_\alpha\sim \Gamma (\alpha, 1).
\end{align*}
Without loss of generality, we assume $\alpha\in\N$ (otherwise consider $\lceil \alpha\rceil$ and $\lfloor \alpha\rfloor$). Then, we can write $X^{(k)}_\alpha$ as a sum of i.i.d. exponentials with unit rate:
\begin{align*}
&X^{(k)}_\alpha \stackrel{\mathcal D}{=}\sum_{i=1}^{\alpha}G_i & G_i\stackrel{\text{i.i.d.}}{\sim}\text{Exp}(1),
\end{align*}
where $\text{Exp}(1)$ is the exponential random variable with unit rate. 
An application of Bernstein's inequality yields
\begin{align*}
\P\left(0.99 \alpha\leq W_k^{\alpha}([0,\alpha])\leq1.01 \alpha\right)\geq 1-2e^{-c_2\alpha},
\end{align*}
where $c_2>0$ is an absolute constant (depending only on $\beta$ and $\lambda$ both of which are equal to $1$). 
Taking a union bound over $k$ yields
\begin{align}
\P\left(0.99\alpha\leq \min_{k}W_k^{\alpha}([0,\alpha])\leq\max_k W_k^{\alpha}([0,\alpha])\leq 1.01\alpha\right)\geq 1-2Ke^{-c_2\alpha}.\label{002}
\end{align}
Combining \eqref{001} and \eqref{002} via a union bound yields that, with probability at least $1-2e^{-c_1(0.99 \alpha)^K}-2Ke^{-c_2\alpha}$,
\begin{align}
&N^\alpha\leq |T|\leq 1.01^{K+1}\alpha^K& |T|\geq 0.99^{K+1}\alpha^K.\label{re1}
\end{align}

A lower bound on $N^\alpha$ requires more refined analysis.
Let 
\begin{align}
&a = \frac{1.01}{0.99^{\frac{K+1}{K}}}(2K\log \alpha)^{1/K}\leq 1.03(2K\log\alpha)^{1/K}&K\geq 2.\label{zsa}
\end{align}
For all sufficiently large $\alpha$, $1\leq a\to\infty$, and
\begin{align}
\alpha\cdot m([a, \infty))&\leq\alpha\cdot m([1, \infty)) = \alpha\int_1^{\infty}\frac{e^{-x}}{x}dx\leq\alpha\int_1^{\infty}e^{-x}dx\leq\alpha\label{hhhh}\\
\alpha\cdot m([a, \infty)) &\geq \alpha\cdot m([a, 1.01a))= \alpha\int_a^{1.01a}\frac{e^{-x}}{x}dx\nonumber\\
&\geq\alpha\int_a^{1.01a}\frac{e^{-x}}{1.01a}dx=\frac{\alpha e^{-a}\left(1-e^{-0.01a}\right)}{1.01a}\geq\frac{0.99\alpha e^{-a}}{a}\geq \sqrt{\alpha}.\label{hhh} 
\end{align}
In this case, a similar Poisson tail estimate as before yields
\begin{align}
&\P\left(0.99\alpha m([a, \infty))\leq \min_k\#\{s\leq \alpha: \Delta X^{(k)}_s\in [a, \infty)\}\leq \max_k\#\{s\leq \alpha: \Delta X^{(k)}_s\in [a, \infty)\}\leq 1.01\alpha m([a, \infty))\right)\nonumber\\
\stackrel{\eqref{hhh}}{\geq}&\ 1-2K e^{-c_1\sqrt{\alpha}}\label{ggg}.
\end{align}

Conditioning $M$ and $|T|$ on the intersection of the events in \eqref{002}, \eqref{re1} and \eqref{ggg}, we write
\begin{align}
N^\alpha = \sum_{s = (s_1, \cdots, s_K)\in\text{supp}(M)}\mathbf 1_{s\in T}\geq\sum_{s\in\mathcal S}\mathbf 1_{s\in T},\label{myNt}
\end{align}
where
\begin{align*}
\mathcal S = \left\{s = (s_1, \cdots, s_K)\in\text{supp}(M): \Delta X_{s_k}^{(k)}\geq a, \;\forall k \in [K]\right\}.
\end{align*}
Each term in the summand in the right-hand side of \eqref{myNt} is a Bernoulli random variable with parameter
\begin{align}
p_s:=1-\left(1-\frac{\prod_{k=1}^K\Delta X^{(k)}_{s_k}}{M([0,\alpha]^K)}\right)^{|T|}&\stackrel{\eqref{002}, \eqref{re1}}{\geq} 1-\left[1-\left(\frac{a}{1.01\alpha}\right)^K\right]^{0.99^{K+1}\alpha^K}\nonumber\\
& = 1-\left(1-\frac{2K\log\alpha}{0.99^{K+1}\alpha^K}\right)^{0.99^{K+1}\alpha^K}\geq 1- \alpha^{-2K}\label{lalaba},
\end{align}
where the last step used the fact that $(1-\frac{1}{x})^x\leq e^{-1}$ for $x>1$.
In particular, for $s\in\mathcal S$, 
\begin{align}
\P\left(\mathbf 1_{s\in T}=0\right) = 1- p_s\stackrel{\eqref{lalaba}}{\leq} \alpha^{-2K}.\label{zzzz}
\end{align}
Since
\begin{align}
\frac{0.95^K(\alpha e^{-1.03(2K)^{1/K}(\log \alpha)^{1/K}})^K}{2K\log\alpha}\stackrel{\eqref{zsa}}{\leq} 0.99^{2K}\left(\frac{\alpha e^{-a}}{a}\right)^K\stackrel{\eqref{hhh}, \eqref{ggg}}{\leq} |\mathcal S|\stackrel{\eqref{hhhh}, \eqref{ggg}}{\leq} 1.01^K \alpha^K,\label{zzz}
\end{align}
taking a union bound over $s\in \mathcal S$ yields that, with probability at least 
\begin{align*}
&1-2e^{-c_1(0.99\alpha )^K}-2Ke^{-c_2\alpha}-2Ke^{-c_3\alpha}-2Ke^{-c_1\sqrt{\alpha}}-\sum_{s\in\mathcal S}(1-p_s)\\
\stackrel{\eqref{zzzz}, \eqref{zzz}}{\geq}&\ 1-2e^{-c_1(0.99\alpha )^K}-2Ke^{-c_2\alpha}-2Ke^{-c_3\alpha}-2Ke^{-c_1\sqrt{\alpha}}-1.01^K\alpha^{-K},
\end{align*}
the following holds: 
\begin{align}
N^\alpha\geq |\mathcal S|\geq \frac{0.95^K(\alpha e^{-1.03(2K)^{1/K}(\log \alpha)^{1/K}})^K}{2K\log\alpha}.\label{re3}
\end{align}

\textbf{Analysis of $D_k^\alpha$.}
Analysis of $D_k^\alpha$ is easy owing to an observation in \cite{tillinghast2021nonparametric}: For $k\in [K]$, conditioned on $W_\ell^\alpha$, $\ell\neq k$, 
\begin{align*}
&D_k^\alpha \sim \text{Poisson}\left(\alpha\psi\left(\gamma_\alpha^{(-k)}\right)\right),
\end{align*}
where 
\begin{align*}
&\gamma_\alpha^{(-k)} = \prod_{\ell\neq k}W_k^{\alpha}([0,\alpha])&\psi(\alpha) = \int_\R (1-e^{-\alpha x})m(dx). 
\end{align*}
By a similar Poisson tail estimate as before,
\begin{align}
\P\left(0.99\alpha \psi(\gamma_\alpha^{(-k)})\leq D_k^\alpha\leq 1.01\alpha \psi(\gamma_\alpha^{(-k)})\right)\geq 1-2e^{-c_1\alpha\psi(\gamma_\alpha^{(-k)})}. \label{poil}
\end{align}
It is easy to check via L'H\^opital's rule that
\begin{align*}
\lim_{\alpha\to\infty}\frac{\psi(\alpha)}{\log \alpha} = \lim_{\alpha\to\infty}\alpha\frac{d}{d\alpha}\psi(\alpha) = \lim_{\alpha\to\infty}\alpha\int_\R xe^{-\alpha x}(1-e^{-\alpha x})m(dx) =  \lim_{\alpha\to\infty}\frac{\alpha^2}{(\alpha+1)(2\alpha+1)} = \frac{1}{2}.
\end{align*}
Hence, for all sufficiently large $\alpha$, 
\begin{align}
0.49\log \alpha\leq\psi(\alpha)\leq 0.51\log \alpha.\label{asymp}
\end{align} 
Thus, conditioned on the event $0.99 \alpha\leq \min_{\ell\neq k}W_\ell^\alpha([0,\alpha])\leq\max_{\ell\neq k} W^\alpha_\ell([0,\alpha])\leq 1.01 \alpha$ (which holds with probability at least as the lower bound in \eqref{002}), \eqref{poil} and \eqref{asymp} together implies that,  for all sufficiently large $\alpha$, 
\begin{align*}
\P\left(0.48(K-1)\alpha\log\left(0.99\alpha \right)\leq D_k^\alpha\leq 0.52(K-1)\alpha\log\left(1.01\alpha \right)\right)\geq 1-2e^{-c_1(0.99\alpha )^K}.
\end{align*}
Taking a union bound over $k$ yields that, for all sufficiently large $\alpha$, with probability at least $1-2Ke^{-c_1(0.99\alpha )^K}$, 
\begin{align}
0.48(K-1)\alpha\log\left(0.99\alpha \right)\leq \min_k D_k^\alpha\leq \max_kD_k^\alpha\leq 0.52(K-1)\alpha\log\left(1.01\alpha \right).\label{re2}
\end{align}
Combining \eqref{re1}, \eqref{re3}, \eqref{re2} and renaming the constants yields the desired result.
\end{proof}



\end{document}